\documentclass{article}

\usepackage[preprint]{paper_style}

\usepackage[utf8]{inputenc} 
\usepackage[T1]{fontenc}    
\usepackage{hyperref}       
\usepackage{url}            
\usepackage{booktabs}       
\usepackage{amsfonts}       
\usepackage{nicefrac}       
\usepackage{microtype}      
\usepackage{xcolor}         

\usepackage{amsthm}
\usepackage{amsmath}

\newtheorem{theorem}{Theorem}[section]

\newtheorem{remark}{Remark}

\usepackage{amsthm}

\newtheorem{definition}{Definition}

\usepackage{cite}  
\usepackage{comment}
\usepackage{graphicx}
\usepackage[linesnumbered,ruled,vlined]{algorithm2e}
\usepackage{algorithmic}

\usepackage{multirow}
\usepackage{dsfont}
\usepackage{color}
\usepackage{float}

\definecolor{yjc}{RGB}{225,0,100}
\definecolor{lxs}{RGB}{138,43,226}

\definecolor{own_pink}{RGB}{217,25,169}
\definecolor{own_blue}{RGB}{0,100,223}

\definecolor{own_pink}{RGB}{217,25,169}
\definecolor{own_blue}{RGB}{0,100,223}

\allowbreak

\usepackage{scalerel,stackengine}
\stackMath
\newcommand\reallywidehat[1]{%
\savestack{\tmpbox}{\stretchto{%
  \scaleto{%
    \scalerel*[\widthof{\ensuremath{#1}}]{\kern-.6pt\bigwedge\kern-.6pt}%
    {\rule[-\textheight/2]{1ex}{\textheight}}
  }{\textheight}%
}{0.5ex}}%
\stackon[1pt]{#1}{\tmpbox}%
}
\newcommand\reallywidecheck[1]{%
\savestack{\tmpbox}{\stretchto{%
  \scaleto{
    \scalerel*[\widthof{\ensuremath{#1}}]{\kern-.6pt\bigwedge\kern-.6pt}%
    {\rule[-\textheight/2]{1ex}{\textheight}}
  }{\textheight}%
}{0.5ex}}%
\stackon[1pt]{#1}{\scalebox{-1}{\tmpbox}}%
}

\usepackage{enumitem}

\usepackage{listings}
\usepackage{color}

\definecolor{codegreen}{rgb}{0,0.6,0}
\definecolor{codegray}{rgb}{0.5,0.5,0.5}
\definecolor{codepurple}{rgb}{0.58,0,0.82}
\definecolor{backcolour}{rgb}{0.95,0.95,0.92}

\lstdefinestyle{mystyle}{
    backgroundcolor=\color{backcolour},   
    commentstyle=\color{codegreen},
    keywordstyle=\color{magenta},
    numberstyle=\tiny\color{codegray},
    stringstyle=\color{codepurple},
    basicstyle=\ttfamily\footnotesize,
    breakatwhitespace=false,         
    breaklines=true,                 
    captionpos=b,                    
    keepspaces=true,                 
    numbers=left,                    
    numbersep=5pt,                  
    showspaces=false,                
    showstringspaces=false,
    showtabs=false,                  
    tabsize=2
}

\usepackage{booktabs}

\usepackage{subcaption}

\usepackage{array}

\usepackage{xcolor} %

\newcommand{\set}[1]{\left\{ #1 \right\}}
\newcommand{\brac}[1]{\left[ #1 \right]}
\newcommand{\norm}[1]{\left \Vert #1 \right\Vert}

\newcommand{\M}{\mathcal{M}}
\newcommand{\E}{\mathcal{E}}
\newcommand{\R}{\mathcal{R}}
\newcommand{\LL}{\mathcal{L}}

\usepackage{adjustbox}

\usepackage{threeparttable}

\DeclareMathOperator*{\argmin}{arg\,min}

\usepackage{adjustbox}
\usepackage{booktabs}
\usepackage{caption}
\usepackage{wrapfig} 

\title{Few-Shot Test-Time Optimization Without Retraining for Semiconductor Recipe Generation and Beyond}

\author{%
  \textbf{Shangding Gu\textsuperscript{1}}\thanks{corresponding to \textit{shangding.gu@berkeley.edu}},~
  \textbf{Donghao Ying\textsuperscript{1}},
  \textbf{Ming Jin\textsuperscript{2}},
  \textbf{Yu Joe Lu\textsuperscript{3}},
  \textbf{Jun Wang \textsuperscript{4}}, \vspace{5pt} \\ 
  \textbf{  Javad Lavaei \textsuperscript{1}}, 
  \textbf{Costas Spanos\textsuperscript{1}} \\ \And
  \textsuperscript{1}UC Berkeley \ \
  \textsuperscript{2}Virginia Tech \ \ 
  \textsuperscript{3}Lam Research \ \  
  \textsuperscript{4}UCL \ \ 
}

\begin{document}

\maketitle

\begin{abstract}

We introduce Model Feedback Learning (MFL), a novel test-time optimization framework for optimizing inputs to pre-trained AI models or deployed hardware systems \textit{without} requiring any retraining of the models or modifications to the hardware. In contrast to existing methods that rely on adjusting model parameters, MFL leverages a lightweight reverse model to iteratively search for optimal inputs, enabling efficient adaptation to new objectives under deployment constraints. This framework is particularly advantageous in real-world settings, such as semiconductor manufacturing recipe generation, where modifying deployed systems is often infeasible or cost-prohibitive. We validate MFL on semiconductor plasma etching tasks, where it achieves target recipe generation in just five iterations, significantly outperforming both Bayesian optimization and human experts. Beyond semiconductor applications, MFL also demonstrates strong performance in chemical processes (e.g., chemical vapor deposition) and electronic systems (e.g., wire bonding), highlighting its broad applicability. Additionally, MFL incorporates stability-aware optimization, enhancing robustness to process variations and surpassing conventional supervised learning and random search methods in high-dimensional control settings. By enabling few-shot adaptation, MFL provides a scalable and efficient paradigm for deploying intelligent control in real-world environments.
\end{abstract}

\section{Introduction}

Semiconductor manufacturing has been a central driver of AI progress, producing increasingly powerful chips that meet large-scale computational demands \citep{kanarik2023human, khan2020ai}. Yet, after manufacturing equipment is deployed, optimizing input parameters (e.g., in etching processes) remains challenging. Concurrently, AI has shown remarkable success across diverse fields—ranging from the Game of Go \citep{silver2016mastering} and protein design \citep{jumper2021highly} to robotics \citep{kroemer2021review, gu2023safe, kumar2024robohive}, autonomous driving \citep{grigorescu2020survey, zhao2024autonomous}, and conversational agents \citep{ouyang2022training, achiam2023gpt}. However, as these models grow in complexity, retraining them after deployment is often computationally expensive or operationally infeasible. This challenge leads to a critical question: 
\begin{center}
    \textbf{How can we adapt the inputs to a deployed model to achieve new objectives, without modifying or retraining the model itself?} 
\end{center}

We address this question with Model Feedback Learning (MFL), which efficiently performs test-time optimization to identify optimal inputs for new targets without modifying the parameters of the deployed model. By iteratively fine-tuning a lightweight reverse model, MFL adapts the system’s behavior to new requirements or corrections.

\begin{figure}
    \centering
    \includegraphics[width=0.99\linewidth]{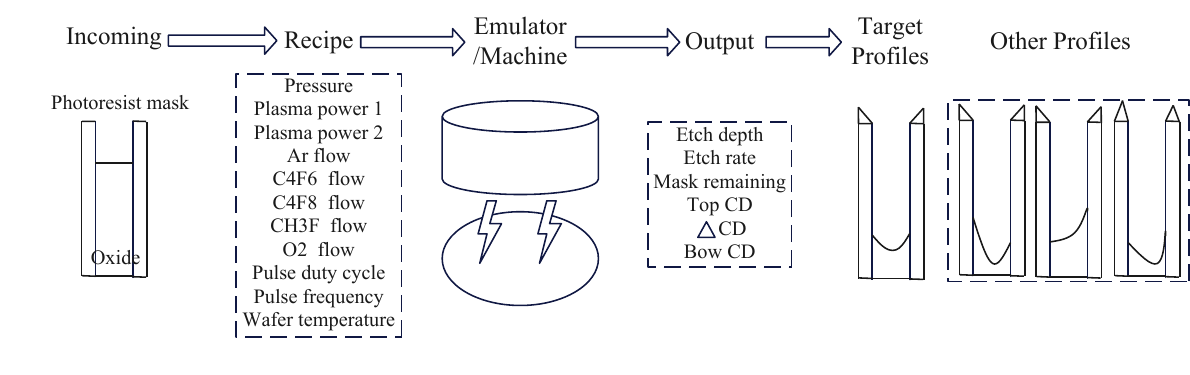}    
    \vspace{-15pt}
    \caption{Schematic of the semiconductor manufacturing process, showing the incoming photoresist mask, recipe inputs (e.g., gas flows, plasma powers, pulsing parameters, and wafer temperature), and the resulting etched profiles \citep{kanarik2023human}. The emulator/machine outputs include etch depth, etch rate, mask remaining, top CD, $\Delta$CD, and bow CD. Deviations (as seen in the mismatch between target and actual profiles) can occur due to process variability.}    
    \label{fig:semiconductor-manufacturing-overview-nature}
\end{figure}

To validate MFL, we apply it to semiconductor manufacturing for chip production \citep{shearn2010advanced}—a high-cost domain that demands nanometer-scale precision \citep{kanarik2023human}. Specifically, we focus on plasma etching recipe generation (Figure~\ref{fig:semiconductor-manufacturing-overview-nature}), where small variations in input parameters (e.g., gas flow rates, plasma power, wafer temperature) can significantly impact critical dimensions and etch depths. As modifying the underlying physical model is prohibitively expensive, effective input optimization is essential. Further details on semiconductor plasma etching are provided in Appendix~\ref{appendix:plasma_etching}. To assess the generality of our approach, we also apply MFL to other applications, including Chemical Vapor Deposition (CVD) \citep{jones2009chemical, sun2021chemical} and wire bonding \citep{chauhan2014copper, zhong2011overview}. In both cases, our method achieves high efficiency and performance, requiring only 5 iterations for CVD and 9 iterations for wire bonding to produce satisfactory results—all while strictly adhering to domain-specific constraints on both input and target variables.

Our key contributions are summarized as follows:
\begin{itemize}[leftmargin=*]
    \item \textbf{Retraining-Free Input Optimization:} Through test-time optimization with a lightweight reverse model that iteratively searches for input values yielding desired outputs, MFL eliminates the need to retrain deployed models, thereby reducing computational costs and enabling real-world deployment in resource-constrained settings.
    \item  \textbf{Applications in Semiconductor Manufacturing and Beyond:} Through comprehensive experiments, MFL demonstrates precise control over etch depth, mask remaining, and critical dimensions in plasma etching. It also generalizes effectively to other domains, including chemical and electronic applications. 
    \item \textbf{Improved Efficiency and Robustness:} MFL with a two-loop learning framework achieves substantial improvements in data efficiency compared to existing methods and demonstrates strong robustness to noisy environments.
\end{itemize}


Existing methods typically rely on training a comprehensive dynamics model of the entire system, which is both computationally expensive and data-intensive. In contrast, our method utilizes a lightweight, task-specific reverse model to enable efficient test-time optimization for identifying optimal inputs given a desired output. Furthermore, when applied to semiconductor manufacturing and other applications, MFL outperforms relevant baselines, highlighting its practical effectiveness and adaptability.

\section{Related Work}
\label{section:related-work}
\paragraph{Physical Model-Based Semiconductor Manufacturing}
Semiconductor processes have long relied on physical models for tasks like etching and deposition \citep{wang2009modeling, wei2009advanced, chang2013novel, nicole2011modeling, monahan2005enabling, shen2017three, may2006fundamentals}. For example, \citep{kim2024application} uses plasma-based virtual metrology to generate etching recipes. However, such methods often require specialized plasma data, which may be unavailable in many scenarios. Similarly, \citep{park2024plasma} proposes a traditional plasma etching method that performs well but depends heavily on prior specific  knowledge, limiting its broader applicability. In related work, \citep{park2024damage} achieves atomic-scale precision with a damage-free plasma source, yet also relies on extensive expert domain insights.

\paragraph{Data-Driven Semiconductor Manufacturing}
Recently, machine learning has gained popularity for semiconductor process optimization \citep{shim2016machine, kim2005prediction, ma2014fast, ma2017fast, myung2021novel, jeong2021bridging, sawlani2024perspectives, hosseinpour2024novel}. For instance, \citep{kanarik2023human} combines Bayesian optimization and human collaboration to optimize etching profiles, while \citep{yao2025etching} uses cascade recurrent neural networks (RNNs) trained on simulation data. Transformer-based approaches \citep{waswani2017attention} have also been explored, such as \citep{wang2024data}’s etching rate prediction augmented by soft sensors. Other methods leverage backpropagation neural networks \citep{chen2020etch, li2012brief} or multiscale modeling with RNNs \citep{xiao2021multiscale}. Although effective, these strategies typically rely on repeated model retraining for each new target, which becomes computationally expensive in practice.

\paragraph{Test-Time Training} Several test-time training methods have been proposed in recent years \citep{sun2020test}, with growing applications in computer vision \citep{sinha2023test} and language models \citep{akyurek2024surprising}. For instance, \citep{sun2020test} explores test-time training without relying on labeled data; however, their method typically requires a large number of samples to achieve optimal performance, limiting its effectiveness in few-shot scenarios. \citep{liu2021ttt} proposes a method that combines online moment matching with offline feature summarization to enhance performance, though its effectiveness may degrade under real-world conditions due to imperfect feature alignment. \citep{bartler2022mt3} presents a meta test-time training framework that integrates meta-learning and self-supervised learning. \citep{gandelsman2022test} introduces a mask encoder for test-time adaptation, and \citep{wang2021tent} adopts entropy minimization to guide the adaptation process. In contrast, our approach based on model feedback learning can efficiently search for target solutions with only a few-shot adaptation steps using a light-weight model (e.g., 7 kB), even under significant distribution shifts.

Our work is most closely related to inverse neural networks \citep{ardizzone2018analyzing, li2018nett}, and Internal Model Control (IMC) \citep{garcia1982internal}. Unlike inverse neural networks, which often involve retraining the original model, our MFL framework avoids retraining by using a lightweight reverse model to adapt inputs efficiently. It also differs from IMC, which heavily depends on accurate system models and lacks a direct feedback mechanism from real-world machines. In contrast, MFL does not require an exact model and can adaptively learn from data, allowing it to tackle real-world process variability. Furthermore, MFL is model-agnostic, easily integrates with existing AI algorithms, and rapidly identifies optimal inputs when process targets shift.

\section{Preliminary}

\begin{figure}[tb]
\centering
  {
\includegraphics[width=0.6\linewidth]{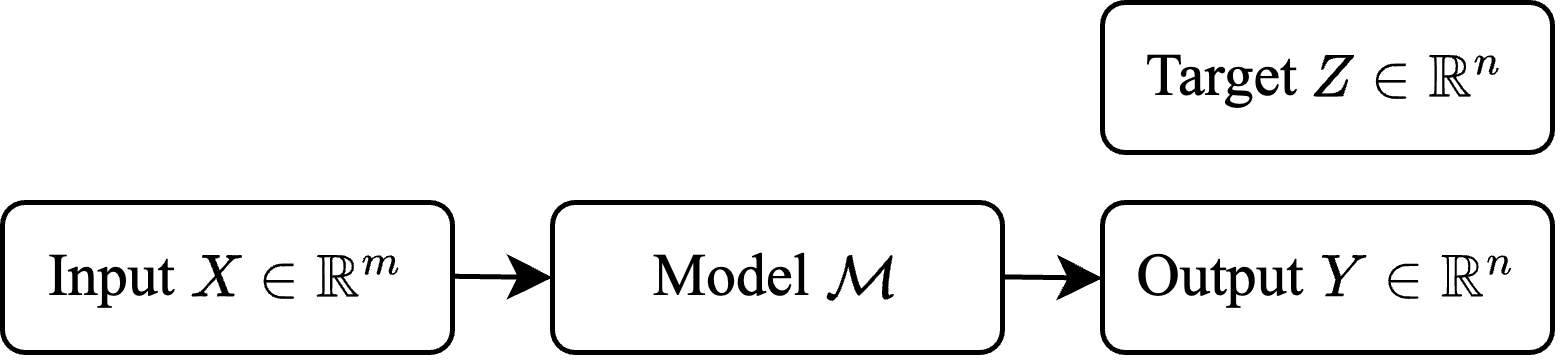}
}      
	\caption{\normalsize Illustration of standard supervised learning.
  }
\label{fig:model-feedback-sl}
\end{figure}

In recent years, supervised learning has emerged as one of the primary data-driven approaches for generating recipes in semiconductor manufacturing \citep{chen2024exploring}. In a standard supervised learning problem, we are usually given the input data $X \in \mathbb{R}^m$ and target data $Z\in \mathbb{R}^{m}$ that follow some underlying distributions\footnote{We use capitalized letters to denote random variables and non-capitalized letters to denote realizations.}.
Typically, the goal of the supervised learning is to train a model $\M$, parameterized by neural networks with some parameter $\theta_{\M} \in \mathbb{R}^{n \times m}$, such that its output $Y\in \mathbb{R}^n$ closely approximates the target $Z$, i.e., $\M(X) = Y \approx Z$ (see Figure \ref{fig:model-feedback-sl} for an illustration).
This can be achieved by training $\M$ in a supervised learning approach, i.e., feed the model with input data $X$ and target data $Z$, then optimize the parameter $\theta_{\M}$ to minimize the difference between $Y$ and $Z$.

However, in some scenarios, particularly after model $\M$ has been deployed, it becomes challenging to retrain the model to adapt to new environments with a different desired target $Z^\prime\in \mathbb{R}^{m}$.
In such cases, instead of modifying the model parameter $\theta_{\M}$ itself, we aim for an alternative approach: finding a suitable input $X^\prime \in \mathbb{R}^{n}$ such that the model $\M$ produces an output $Y^\prime \in \mathbb{R}^{m}$ that closely approximates the target $Z^\prime$, i.e., $\M(X^\prime) = Y^\prime \approx Z^\prime$.
This approach, known as input optimization, allows tuning the model's input without retraining the model itself.
It is particularly useful in deployed systems or when computational resources are limited, e.g., in semiconductor manufacturing where the data is expensive. 
Formally, our problem can be cast as:
\begin{align}
\label{eq:problem-formulation-key}
\min_{X^\prime\in \mathbb{R}^n} \text{error}(\M(X^\prime), Z^\prime),
\end{align}
where $\text{error}(\cdot,\cdot)$ is some metric quantifying the discrepancy between the model’s output and the target.
Here, we highlight the difference between our problem \eqref{eq:problem-formulation-key} and standard optimization problems.
\begin{itemize}[leftmargin = *]
    \item \textbf{Standard optimization:} The goal is to find a specific solution  $x^\star$  that minimizes some given loss function $f(\cdot)$. When $f(\cdot)$ is a black-box function—difficult to access or evaluate—techniques like Bayesian optimization can be employed to efficiently solve the problem \citep{frazier2018tutorial}.
    \item \textbf{Our problem \eqref{eq:problem-formulation-key}:} The focus is on identifying suitable inputs  $X^\prime$ that correspond to a distribution of new targets $Z^\prime$. 
    Instead of finding a single optimal solution, the objective is more close to approximating the inverse mapping $\M^{-1}(\cdot)$ for the entire target group $Z^\prime$. 
    Notably, Eq.~\eqref{eq:problem-formulation-key} reduces to a standard optimization problem only in the special case of identifying a single input $x^\prime$  such that $\M(x^\prime) = z^\prime$ for a specific target $z^\prime$.
\end{itemize}

\section{Method}
\label{section:method}

\subsection{Two-Loop Feedback Learning}\label{subsec:two_loop_MFL}

To address the optimization problem formulated in Eq. \eqref{eq:problem-formulation-key}, we propose Model Feedback Learning (MFL), a novel framework to identify an ideal input $X^\prime$ that minimizes the discrepancy between the model output and desired target.
The MFL framework comprises two key components:
\vspace{-5pt}
\begin{enumerate}[leftmargin = 4.7ex]
    \item[\textbf{(I)}] Using neural networks to build an emulator $\mathcal{E}$ and a reverse emulator $\mathcal{R}$.
    The emulator $\mathcal{E}$ is used to replicate the behavior of the machine model $\mathcal{M}$, with its parameter remaining fixed during subsequent stages.
    The reverse model $\mathcal{R}$ is designed to map desired targets to corresponding inputs, and it will be specifically optimized by the mechanism described below.
    \item[\textbf{(II)}] A practical two-loop training process for the reverse emulator $\mathcal{R}$, designed for robust test-time performance in real-world applications, as depicted in Figure \ref{fig:model-feedback-sl-robust}.
    \textbf{Loop A}: Pre-training $\mathcal{R}$ using an emulator model $\mathcal{E}$ through an iterative optimization process.
    \textbf{Loop B}: Training $\mathcal{R}$ using the machine model $\mathcal{M}$ through an iterative optimization process.    
    Note that in these two iterative processes, we only update the reverse emulator $\mathcal{R}$ and do not modify the emulator and machine models.
\end{enumerate}

\begin{figure}
    \centering
  {
\includegraphics[width=0.8\linewidth]{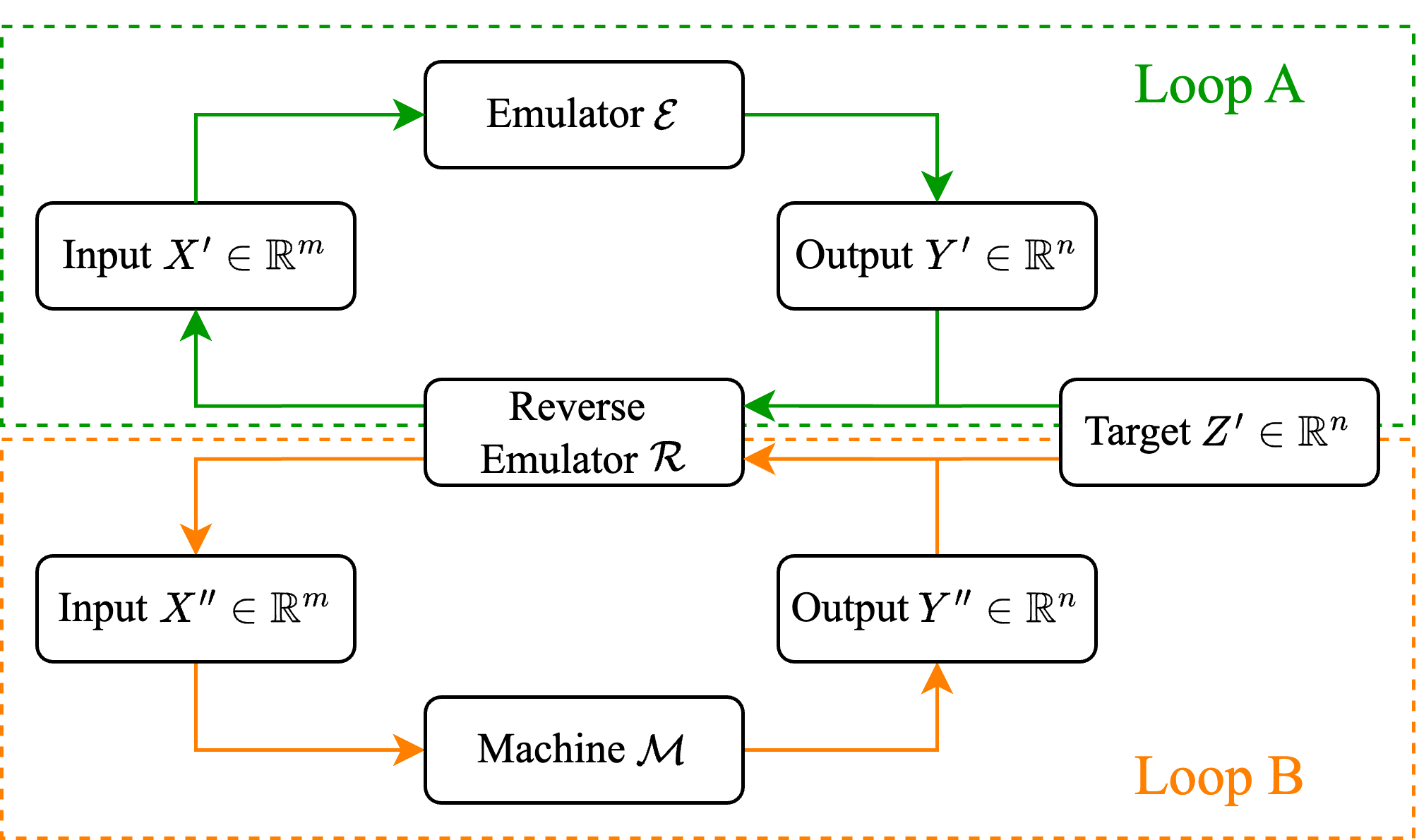}
}    
\caption{\normalsize 
    Two-loop training process for the reverse emulator model $\R$: Pre-train $\R$ using an emulator model $\E$ in Loop A and tune $\R$ using the machine model $\M$ in Loop B.
    } 
  \label{fig:model-feedback-sl-robust}
\end{figure}

Let $\theta$ be the parameter of the reverse emulator, and we use $\R_{\theta}$ to refer to the parameterized reverse model.

Below, we detail the training process of the reverse model $\R$ in Loop A.
In each iteration $t\geq 0$, the current reverse model $\R_{\theta^t}$ is applied to the target data $Z^\prime$ to produce an approximated input, $X_t^\prime = \R_{\theta^t}(Z^\prime)$, where $\theta^t$ denotes the current parameter of $\R$.
This input is then passed through the emulator $\E$ to compute the corresponding output, $Y_t^\prime = \E(X_t^\prime) = \E(\R_{\theta^t}(Z^\prime))$.

The discrepancy $Z^\prime - Y_t^\prime$ arises from two sources: (1) the approximation error of the emulator $\E$ relative to the machine model $\M$, and (2) the suboptimality of the reverse model $\R_{\theta^t}$ relative to the ideal reverse model. 
To better align the reverse model with the emulator $\E$, we define the loss function as the mean-squared error (MSE) between the target data $Z^\prime$ and the predicted output $Y_t^\prime$.
Given the realized samples $\{z_j^\prime\}_{j=1}^{n^\prime}$ of $Z^\prime$, $\{x_{t,j}^\prime\}_{j=1}^{n^\prime}$ of $X^\prime_t$, and $\{y_{t,j}^\prime\}_{j=1}^{n^\prime}$ of $Y_t^\prime$, the MSE at iteration $t$ is expressed as:
\begin{equation}\label{eq: loss_main}
\hspace{-3pt}\LL(\theta^t) := \frac{1}{n^\prime} \sum_{j=1}^{n^\prime} \|z_j^\prime - y_{t,j}^\prime\|^2= \frac{1}{n^\prime} \sum_{j=1}^{n^\prime} \|z_j^\prime - \E(x_{t,j}^\prime)\|^2\hspace{-3pt},
\end{equation}
where $\|\cdot\|$ denotes the $L^2$-norm.
Since the output $Y_t^\prime$ is given by $Y_t^\prime = \E(X_t^\prime) = \E(\R_{\theta^t}(Z^\prime))$ for $t \geq 0$, the loss $\LL(\theta^t)$ is indeed a function of $\theta^t$.
To minimize this loss, we employ gradient descent to iteratively update $\theta$. For each iteration $t \geq 0$, the parameter update rule is given by:
\begin{equation}\label{eq: theta_update}
\theta^{t+1} = \theta^t - \alpha^t \cdot \nabla_\theta \LL(\theta^t),
\end{equation}
where $\alpha^t$ is the learning rate.
This optimization process progressively adjusts $\theta$ to reduce the discrepancy between $\R$ and the ideal reverse model at each iteration.
We can use the chain rule to expand the gradient $\nabla_\theta \LL(\theta^t)$ as follows 
\begin{equation}\label{eq: grad_main}
\hspace{-3pt}\nabla_\theta \LL(\theta^t) \hspace{-2pt}=\hspace{-2pt} \frac{2}{n^\prime}\sum_{j=1}^{n^\prime}\hspace{-2pt}\brac{\frac{\partial \R_{\theta^t}(z_j^\prime)}{\partial \theta}}^\top\hspace{-5pt}\cdot\hspace{-1pt} \brac{\frac{\partial \E(x^\prime_{t,j})}{\partial x}}^\top\hspace{-6pt}\cdot (y_{t,j}^\prime \hspace{-2pt}-\hspace{-1pt} z_j^\prime),\hspace{-1pt}
\end{equation}
where $\frac{\partial \R_{\theta^t}(z_j^\prime)}{\partial \theta}$ and $\frac{\partial \E(x^\prime_{t,j})}{\partial x}$ are Jacobian matrices with dimensions $m\times p$ and $n\times m$, respectively.

It is important to recognize that the training process in Loop A primarily aligns the reverse model $\R$ with the emulator model $\E$ rather than the true machine model $\M$. 
While this alignment is an essential first step, it is not sufficient to fully capture the behavior of $\M$ due to discrepancies between $\E$ and $\M$. 
These discrepancies, as illustrated in Figure~\ref{fig:linear-nonlinear-settings}, arise because $\E$ is trained to approximate $\M$ and may not perfectly replicate its behavior.

\begin{figure}
    \centering
    \vspace{-5pt} 
\centering
\hspace{-10pt}
\subcaptionbox{Linear Setting.}{
    \includegraphics[width=0.45\linewidth]{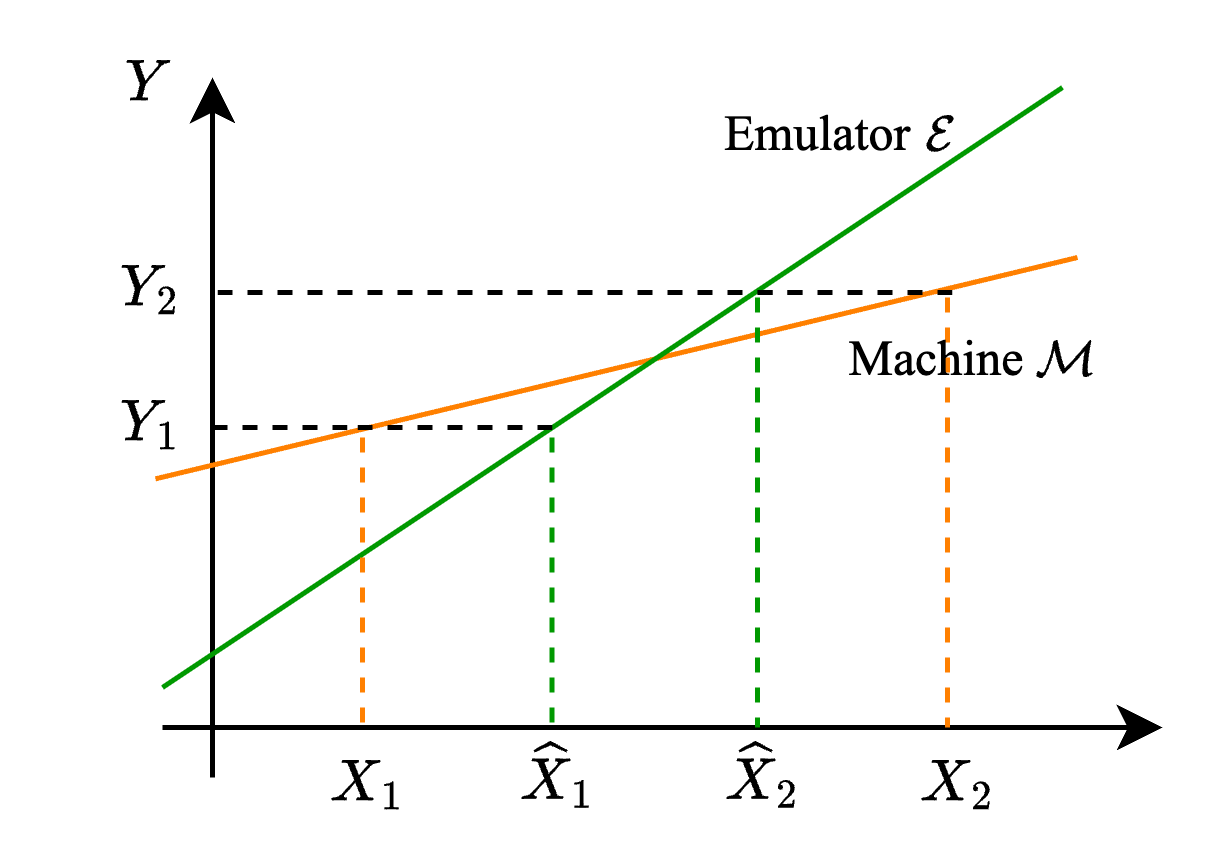}
    }
    \hspace{-1pt}
    \subcaptionbox{Nonlinear Setting.}{
    \includegraphics[width=0.45\linewidth]{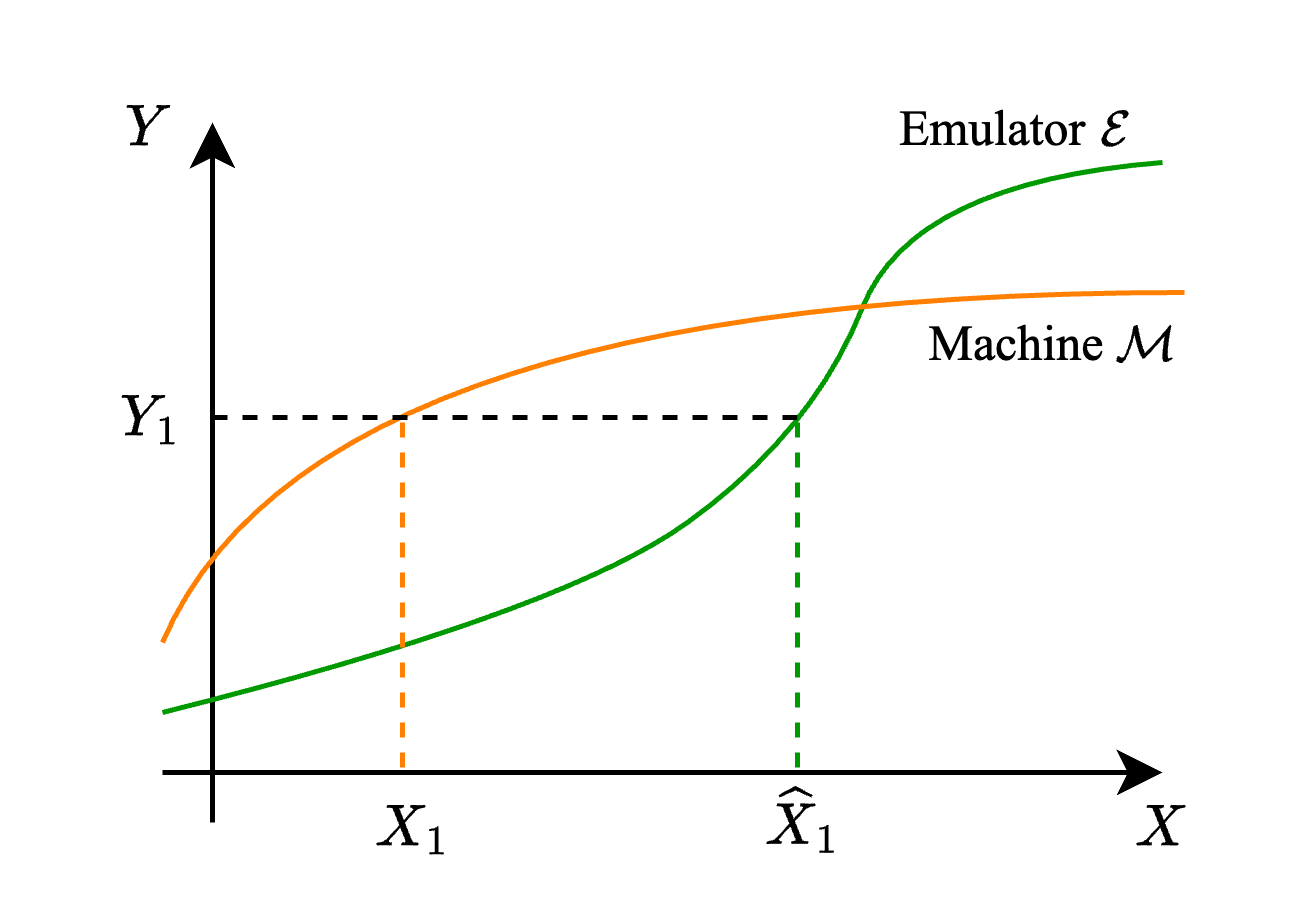}
    } 
\vspace{-0pt}
  \caption{Illustration for the approximation errors of the emulator $\E$. For a given target $Y_i$, $X_i$ denotes the corresponding machine input, and $\widehat{X}_i$ denotes the emulator input.
}
\vspace{-13pt}
\label{fig:linear-nonlinear-settings}
\end{figure}

To address this issue, we introduce Loop B, which further refines the reverse model $\R$ to align with the true machine model $\M$. 
In this phase, $\R$ and $\M$ interact directly through a training protocol mirroring Loop A’s structure.
Particularly, the loss function $\LL(\theta)$ and its gradient are formed by replacing the emulator $\E$ with the true machine $\M$ in Eqs. \eqref{eq: loss_main} and \eqref{eq: grad_main}.
With the pre-training in Loop A, Loop B requires fewer iterations, as $\R$ already achieves a reasonable approximation of the ideal reverse. 
The primary purpose of Loop A is to use the emulator $\E$ to reduce the computational burden associated with direct interactions with $\M$, thereby enhancing efficiency and conserving resources. 
This two-loop framework ensures that $\R$ is both computationally efficient and sufficiently accurate for practical applications, striking a balance between approximation and precision.

\subsection{Conservative Learning}\label{subsec:conservative_learning}
To enhance learning stability, we propose a conservative update mechanism in the later stages of the training process, based on the concept of model sensitivity defined below.

\begin{definition}[Model sensitivity]
The sensitivities of emulator model $\E$ and machine model $\M$ at input $x\in \mathbb{R}^n$ are respectively defined as
\begin{equation}\label{eq:definition-learning-sensitivity}
s_{\E}(x) := \norm{\frac{\partial \E(x)}{\partial x}},\quad s_{\M}(x) := \norm{ \frac{\partial \M(x)}{\partial x} },
\end{equation}
where $\norm{\cdot}$ denotes to the induced $L^2$-norm for matrices.
\end{definition}
Intuitively, $s_{\E}(\cdot)$ and $s_{\M}(\cdot)$ quantify how rapidly the model outputs change in response to variations in the input. 
High sensitivity indicates that even small changes in the input could lead to significant fluctuations in the output.

During the later stages of Loop A (using emulator $\E$) and Loop B (using machine model $\M$), we perform assessments on the model sensitivity before each parameter update to determine the appropriate learning rate. 
If the sensitivity to the given inputs is within the safety range, we still proceed with the standard learning rate $\alpha_1$.
Otherwise, if the sensitivity exceeds a predefined threshold, we adopt a conservative learning rate $\alpha_2$, where $\alpha_2 < \alpha_1$ (see Lines 8–12 and 20–24).

The benefits of this conservative learning paradigm are twofold. 
First, it prevents violent model fluctuations during later training stages, promoting smoother convergence and mitigating instability. 
By avoiding abrupt updates in high-sensitivity regions, it also helps circumvent fluctuations caused by corner cases, where repeated overcorrections could otherwise destabilize the optimization process. 
Second, this paradigm enhances training robustness by avoiding large errors when slight input variations occur in sensitive ranges of the models. 
This is particularly critical in applications like semiconductor manufacturing, where exceeding sensitivity thresholds can lead to significant disruptions. 
By maintaining controlled updates, the conservative approach ensures reliable test-time performance and precision, even in scenarios with tight operational constraints.

\subsection{Practical Algorithm}

In this section, we present the complete algorithm, shown in Algorithm~\ref{alg:algorithm-mfl}, which consists of three key stages.

\begin{itemize}[leftmargin=*]
    \item First, an emulator $\E$ is trained using supervised learning to approximate the true machine model $\M$ (Line 2). 
This step creates a computationally efficient proxy for the machine model, enabling subsequent stages to leverage $\E$ for initial training. 
To improve the robustness of $\E$, we employ the technique of domain randomization.
Specifically, instead of directly using the data pairs $\set{(x_i, z_i)}_{i=1}^n$ during training, we introduce zero-mean Gaussian noise to the inputs $\set{x_i}_{i=1}^n$. 
This deliberate perturbation allows $\E$ to better align with real-world scenarios, where inputs may contain inherent noise, thereby enhancing its stability and ensuring more reliable performance in practical applications.

\item Next, in Loop A (Lines 3–14), the reverse model $\R$ is pre-trained with the emulator $\E$ for $T$ iterations. 
This stage serves as an approximate training phase, efficiently initializing $\R$ by aligning it with $\E$. 
Starting from the $T_0$-th iteration ($T_0 <T)$, the model sensitivity of $\E$ is evaluated at each step to determine the appropriate learning rate. 
If the emulator exhibits high sensitivity to the current inputs, a more conservative learning rate $\alpha_2$ is adopted (Lines 8–12). 
Otherwise, the standard learning rate $\alpha_1$ is used, ensuring a balance between stability and training efficiency.

\item Finally, in Loop B (Lines 15–26), the reverse model $\R$ is further refined by directly interacting with the machine model $\M$ for $\tau$ iterations. 
Similar to Loop A, the conservative learning paradigm is incorporated, with sensitivity assessments determining whether to use the standard learning rate $\alpha_1$ or the conservative learning rate $\alpha_2$ (Lines 20–24). 
Since the pre-training in Loop A effectively aligns $\R$ with the emulator $\E$, the required number of iterations $\tau$ is typically much smaller than $T$, underscoring the computational efficiency and resource-saving benefits of this two-stage training approach.
\end{itemize}

\begin{algorithm}[!htb] 
\caption{\textbf{MFL} -- A Framework of Model Feedback Learning.}
\label{alg:algorithm-mfl}
\begin{algorithmic}[1]
\STATE \textbf{Inputs}: 
machine $\mathcal{M}$; reverse model $\mathcal{R}$ with initial parameter $\theta^0$; emulator training data $\{(x_i,z_i)\}_{i=1}^{n}$; target data $\{z'_j\}_{j=1}^{n^\prime}$ for the new environment; learning rates $\alpha_1$ and $\alpha_2$ with $\alpha_1 > \alpha_2$; periods number $T$, $T_0$, $\tau$, $\tau_0$.
\STATE Construct emulator model $\mathcal{E}$ using $\{(x_i,z_i)\}_{i=1}^{n}$ via supervised learning with domain randomization.
\STATE \textcolor{gray}{ \textbf{// Loop A: Train reverse model $\R_\theta$ using the emulator model $\mathcal{E}$ for $T$ iterations.}}
\FOR{$t = 0, \dots,T-1$} 
\STATE Compute $x'_{t,j}$ = $\R_{\theta^t}(z'_j)$, $\forall j = 1, \ldots, n'$.
\STATE Compute $y_{t,j}' = \mathcal{E}(x_{t,j}')$, $\forall j = 1, \ldots, n'$.
\STATE Compute MSE $\LL(\theta^t)= \frac{1}{n'} \sum_{j=1}^{n'} (z'_j - y_{t,j}')^2$.
\IF{$t\geq T_0$ and $\frac{1}{n^\prime}\sum_{j=1}^{n^\prime} s_{\E}(x^\prime_{t,j})\geq \delta$
}
\STATE Choose conservative learning rate $\alpha^t = \alpha_2$.
\ELSE
\STATE Choose standard learning rate $\alpha^t = \alpha_1$.
\ENDIF
\STATE Update reverse model $\theta^{t+1}=\theta^{t} - \alpha^t\cdot\nabla_{\theta} \LL(\theta^t)$.
\ENDFOR{}
\STATE \textcolor{gray}{\textbf{ // Loop B: Train reverse model $\R_\theta$ using the machine model $\mathcal{M}$ for another $\tau$ iterations.}} 
\FOR{$h = 0, \dots,\tau-1$}
\STATE Compute $x_{h,j}'' = \R_{\theta^{T+h}}(z_j')$, $\forall j = 1, \ldots, n'$.
\STATE Compute $y_{h,j}'' = \mathcal{M}(x_{h,j}'')$, $\forall j = 1, \ldots, n'$.
\STATE Compute MSE $\LL(\theta^{T+h})= \frac{1}{n'} \hspace{-1pt}\sum_{j=1}^{n'} (z'_j - y_{h,j}'')^2$.
\IF{$h \geq \tau_0$ and $\frac{1}{n^{\prime}}\sum_{j=1}^{n^{\prime}} s_{\M}(x^{\prime\prime}_{h,j})\geq \delta$ 
}
\STATE Choose conservative learning rate $\alpha^{T+h} = \alpha_2$.
\ELSE
\STATE Choose standard learning rate $\alpha^{T+h} = \alpha_1$.
\ENDIF
\STATE Update reverse model $\theta^{T+h+1} =\theta^{T+h} - \alpha^{T+h} \cdot\nabla_{\theta} \LL(\theta^{T+h})$.
\ENDFOR{}
\STATE \textbf{Outputs}: Reverse model $\R_{\theta^{T+\tau}}$.
\end{algorithmic}
\end{algorithm}

\begin{remark}[Convergence of Algorithm \ref{alg:algorithm-mfl}]
Under mild regularity conditions, we can demonstrate that Algorithm \ref{alg:algorithm-mfl} converges to a stationary point in the long run.
A more detailed discussion can be found in Appendix \ref{app: convergence}.
\end{remark}

\begin{remark}[Broader application of MFL]
The proposed method provides a versatile framework that encompasses various learning paradigms. While the primary focus of this paper is on the supervised learning setting, the MFL framework can be adapted to scenarios such as online learning, offline learning, and reinforcement learning. 
The exploration of these extensions is deferred to future work.
\end{remark}

\section{Numerical Experiments}
We evaluate the proposed MFL framework for semiconductor manufacturing, specifically focusing on plasma etching recipe generation. Our experiments follow this setup: (1) recipe generation experiments to assess the accuracy performance of our method, (2) benchmarking against state-of-the-art (SOTA) baselines and human expert performance, (3) comparison with other machine learning-based methods, (4) robustness analysis under different conditions, and (5) ablation studies on domain randomization and pre-trained models. Due to the prohibitive costs of real-world fabrication trials, the experiments are conducted in simulation; however, our method can be seamlessly extended to real-world experiments when real-world data becomes available. The experimental settings are provided in Appendix \ref{appendix:experiment-settings}.

\subsection{Recipe Generation}

\begin{table*}[t]
\centering
\vspace{-10pt}
\begin{minipage}[t]{0.3\linewidth}
\centering
\vspace{12pt}
\caption{Comparison of model outputs and targets.}
\hspace{-5pt}
\scriptsize{
\begin{adjustbox}{width=0.9\textwidth,center}
\begin{tabular}{@{}lcc@{}}
\toprule
\textbf{Index} & \textbf{Output} & \textbf{Target} \\
\midrule
Etch depth     & 2255.55 & 2250–2750 \\
Etch rate      & 109.90  & $\geq$100 \\
Mask remain    & 358.95  & $\geq$350 \\
Top CD         & 198.80  & 190–210 \\
$\Delta$CD     & 10.04   & -15–15 \\
Bow CD         & 198.52  & 190–210 \\
\bottomrule
\end{tabular}
\end{adjustbox}
}
\label{table:updated_means_recipe}
\end{minipage}
\hspace{0.04\linewidth}
\begin{minipage}[t]{0.6\linewidth}
\centering 
\caption{Input parameter search ranges. ``SE" denotes senior engineers, and ``JE" represents junior engineers.}
\vspace{-6pt}
\scriptsize{
\begin{adjustbox}{width=0.93\textwidth,center}
\begin{tabular}{@{}lccccc@{}}
\toprule
\textbf{Input Parameters} & \textbf{Unconstrained} & \textbf{SE} & \textbf{JE} & \textbf{Ours} \\
\midrule
Pressure (mT)             & 5 -- 120        & 12 -- 30        & 20 -- 38        & 21.86 \\
Power 1 (W)               & 0 -- 29000      & 4000 -- 15000   & 14000 -- 25000  & 16234.0 \\
Power 2 (W)               & 0 -- 10000      & 1000 -- 7000    & 2000 -- 8000    & 4852.4 \\
Ar Flow (sccm)            & 0 -- 1000       & 100 -- 400      & 300 -- 600      & 253.47 \\
C\textsubscript{4}F\textsubscript{8} Flow & 0 -- 100 & 20 -- 60 & 40 -- 80 & 42.99 \\
C\textsubscript{4}F\textsubscript{6} Flow & 0 -- 100 & 22 -- 66 & 48 -- 96 & 37.40 \\
CH\textsubscript{4} Flow  & 0 -- 20         & 0 -- 5          & 3 -- 8          & 11.60 \\
O\textsubscript{2} Flow   & 0 -- 50         & 20 -- 50        & 20 -- 50        & 18.60 \\
Pulse Duty Cycle (\%)     & 10 -- 100       & 20 -- 60        & 30 -- 70        & 51.46 \\
Pulse Frequency (Hz)      & 500 -- 2000     & 1000            & 1000            & 989.85 \\
Temperature (°C)          & -15 -- 80       & 20 -- 45        & 25 -- 50        & 35.97 \\
\bottomrule
\end{tabular}
\end{adjustbox}
}
\label{table:input_parameters-constraints-sample}
\end{minipage}
\vspace{-10pt}
\end{table*}

In semiconductor manufacturing, recipe accuracy is critical for the etching process, as even a slight deviation in the recipe can result in significant variations in semiconductor fabrication and chip performance. In the experiments, we use a dataset generated from Gaussian sampling, designed to reflect the properties of semiconductor recipes \citep{kanarik2023human}. Our method successfully achieves the targets, enabling precise recipe generation. As shown in Table 1, we compare the outputs of our model with the target ranges \citep{kanarik2023human} for key process metrics, including etch depth, etch rate, mask remaining, top CD, $\Delta$CD, and bow CD. The model outputs align closely with the specified targets, demonstrating the effectiveness of our approach in generating precise and reliable recipes.

Specifically, for etch depth, the model achieves a value of 2255.5466, which is well within the target range of 2250 to 2750. Similarly, the etch rate (109.8992) meets the minimum threshold of $\geq$100, and the mask remaining (358.9485) meets the target of $\geq$350. The top CD and bow CD values (198.7954 and 198.5240, respectively) fall within the range of 190 to 210, and $\Delta$CD (10.043) satisfies the allowable range of -15 to 15. The outputs remain within desired target ranges, as shown in Table \ref{table:updated_means_recipe}, our method achieves accurate results using only \textbf{5 iterations} within the machine loop. Unlike prior approaches that require model retraining or extensive tuning, our approach directly optimizes the recipe in real-time, making it practical for high-throughput semiconductor manufacturing systems. Note, all values are in nanometers (nm), except for the etch rate, which is expressed in nanometers per minute (nm/min). CD refers to the critical dimension, and $\Delta$CD represents the difference between the top and bottom CD of the wafer profile.

These results highlight the accuracy of our model in meeting the semiconductor manufacturing specifications, paving the way for more reliable recipe generation in advanced manufacturing workflows. Our model input also meets the input constraints, which are shown in Table \ref{appendix-table:input_parameters-constraints}, in Appendix \ref{appendix:plasma_etching}.

 \subsection{Comparison with Bayesian Optimization Methods and Human Efforts}

In real-world applications, efficiency is paramount. In our experiments, we compared the efficiency performance of our proposed approach with Bayesian optimization and human-driven strategies for identifying the correct inputs to generate semiconductor recipes.  We present a comparative analysis of the proposed MFL framework, the method introduced by \citep{kanarik2023human} from Lam Research, and traditional human efforts, focusing on the number of iterations required for recipe generation. The Lam Research method incorporates Bayesian optimization and Gaussian approaches with human collaboration. 
Additionally, we conducted experiments using only Bayesian optimization without human efforts for recipe generation and observed that, under the same experimental conditions, it is unable to generate correct recipes.

The experimental results show that MFL significantly outperforms both the Lam Research approach and human-driven methods in terms of efficiency, demonstrating the ability to achieve the correct recipe inputs for optimal semiconductor manufacturing outputs. Specifically, MFL requires only five iterations to generate recipes, significantly fewer than the Lam Research method, which typically requires at least 20 iterations, and far less than the 84 iterations often needed when performed manually by a senior engineer (The results of the Lam Research method and human-driven methods are sourced from \citep{kanarik2023human}). Furthermore, Table~\ref{table:input_parameters-constraints-sample} compares the input parameters generated by our model with those utilized by senior engineers (SE) and junior engineers (JE) \citep{kanarik2023human}. The recipe produced by MFL strictly adheres to the defined constraints, demonstrating its capacity to generate valid and accurate process parameters.

This efficiency highlights the superiority of MFL in accelerating the recipe generation process, indicating its potential to minimize human intervention and reduce computational time. The integration of automated feedback mechanisms in MFL eliminates the reliance on extensive human collaboration, as required in the Lam Research approach while maintaining high accuracy and effectiveness.

\subsection{Comparison with Other Machine Learning Methods}

Next, we evaluate the performance of the MFL framework in semiconductor recipe generation by comparing it with traditional methods, including the \textit{supervised learning approach} \citep{lecun2015deep, oehrlein2024future} and the \textit{Large-Scale Random Search with Local Refinement (LSRS-LR)}. This comparative analysis provides deeper insights into the effectiveness of our proposed method.

\begin{figure}[htbp]
  \centering
  \begin{minipage}[b]{0.47\linewidth}
    \centering
    \includegraphics[width=\linewidth]{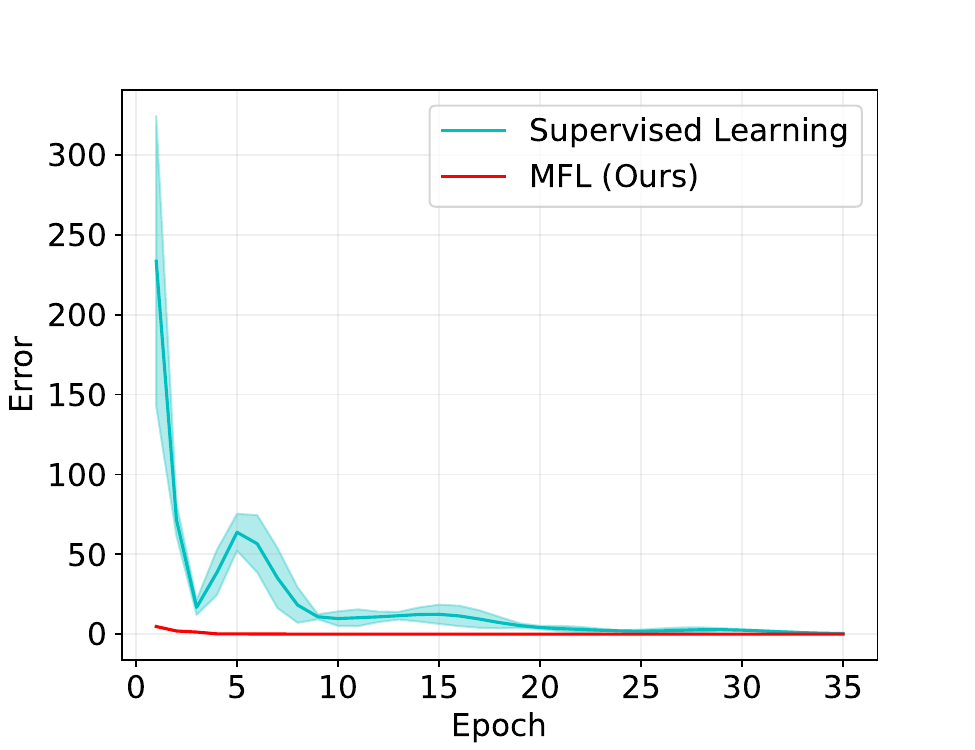}
    \caption{Comparison of MFL and supervised learning in high-dimensional space training, conducted across three random seeds. The y-axis refers to the output error per epoch in a six-dimensional target space.}
    \label{fig:mfl-compared-supervised-learning}
    \vspace{-2pt}
  \end{minipage}
  \hspace{-10pt}
  \hspace{0.05\linewidth}
  \begin{minipage}[b]{0.47\linewidth}
    \centering
    \includegraphics[width=\linewidth]{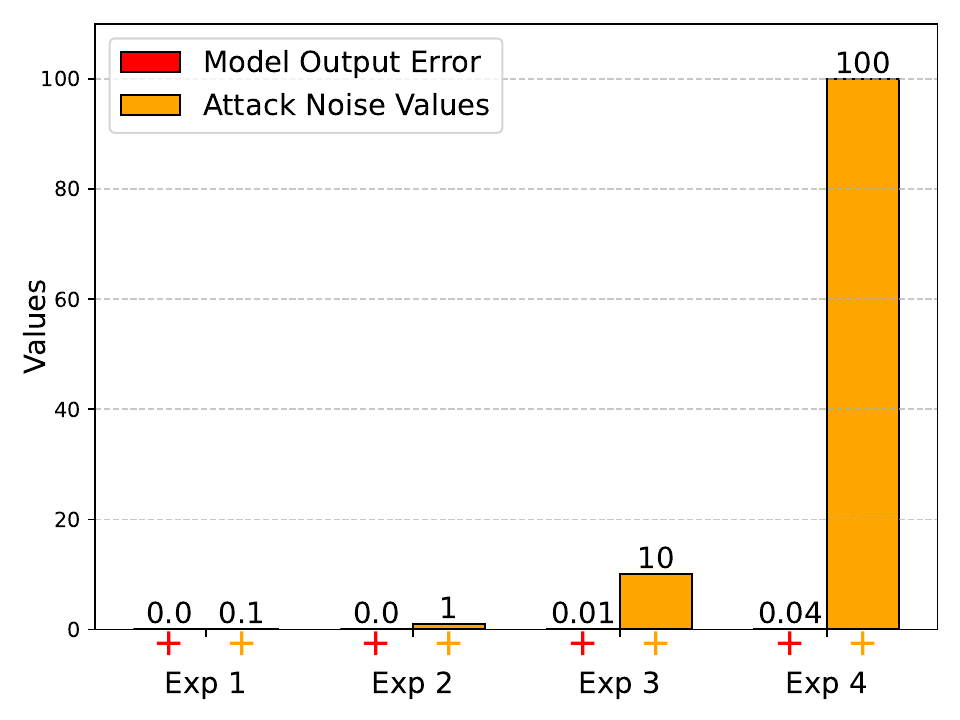}
    \caption{Robust evaluation: comparison of model output errors against different attack noise levels across four experimental scenarios.}
    \label{exp-fig:robust-evaluation}
  \end{minipage}
\end{figure}

As shown in Figure~\ref{fig:mfl-compared-supervised-learning}, we analyze the error reduction patterns during model training, observing clear differences between the two methods. Specifically, the supervised learning approach (indicated in cyan) starts with a high error rate, peaking significantly during the initial epochs. Although the error decreases as training progresses, this reduction is characterized by notable fluctuations. Moreover, the convergence process is gradual, indicating challenges in achieving efficient error minimization within this high-dimensional space. These observations suggest that traditional methods, such as supervised learning, may struggle to adapt effectively in scenarios with high-dimension space. In contrast, the MFL approach (indicated in red) exhibits a stable and consistently low-error trajectory throughout the training process. From the beginning, our method demonstrates its ability to efficiently optimize the model's performance, with minimal error fluctuations. This stability highlights the effectiveness of incorporating model feedback to guide the optimization process, allowing MFL to achieve superior generalization. The results suggest that MFL is particularly well-suited for real-world and high-dimensional tasks where maintaining stability and efficiency is critical.

Furthermore, we develop the LSRS-LR algorithm for semiconductor manufacturing recipe generation and conduct comparative experiments between {MFL} and LSRS-LR. The {LSRS-LR} algorithm optimizes input parameters for a pre-trained emulator model in two stages. First, a \textit{large-scale random search} selects the top $K$ candidates based on MSE loss. Then, \textit{local refinement} iteratively optimizes these candidates while enforcing constraints. This two-stage approach enables the selection of the optimal input $x^*$, making LSRS-LR a viable strategy for complex optimization tasks in semiconductor manufacturing. Experimental results indicate that {MFL} achieves higher accuracy and precision than LSRS-LR in the plasma etching process. For further details on the LSRS-LR algorithm and experimental results, please refer to Algorithm \ref{algorithm:random-search-local-refinement} and Table \ref{tab:etch_data_semiconductor_manufacturing_LSRS-LR} in Appendix \ref{appendix-sec:random-search-optimization}.

Overall, this comparison highlights the significant advantages of MFL over traditional supervised learning and LSRS-LR methods in high-dimensional settings. Particularly, MFL achieves lower error rates, faster convergence, and greater stability, making it ideal for real-world applications that demand dependable model performance under varying conditions.

\subsection{Robustness Evaluation}

\begin{figure}
    \centering
    \subcaptionbox{Minor target shift of 0.1.\label{fig:mfl-sl-target-shift0.1-1-5a}}  
  {\includegraphics[width=0.45\linewidth]{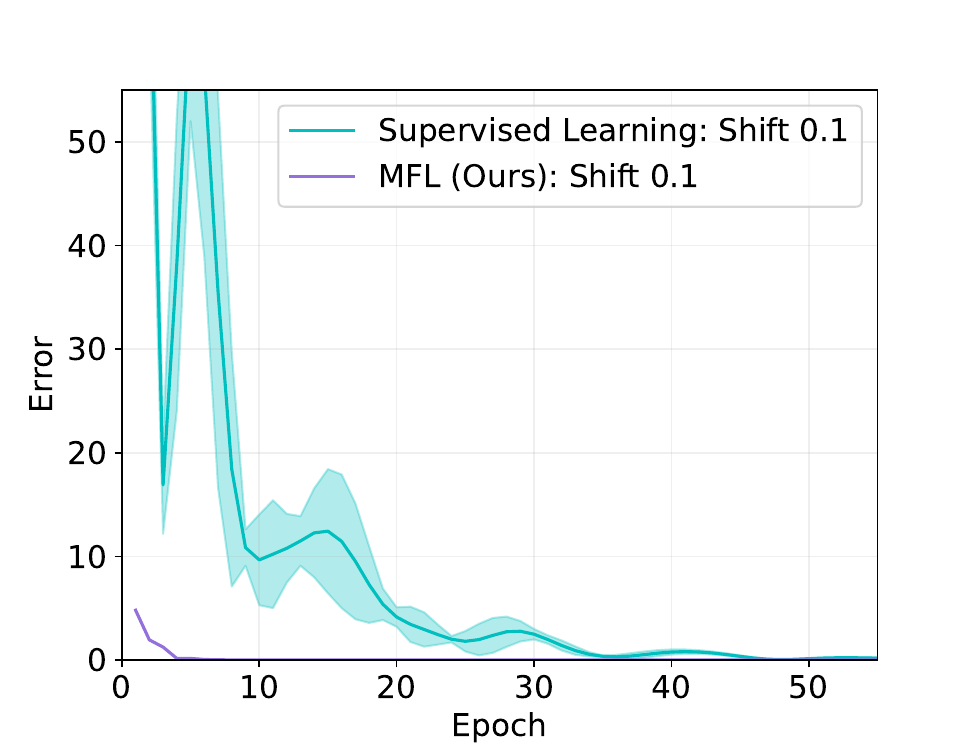}
  }
  \hspace{-10pt}
      \subcaptionbox{Substantial target shift of 5.\label{fig:mfl-sl-target-shift0.1-1-5b}}
  {
     \includegraphics[width=0.45\linewidth]{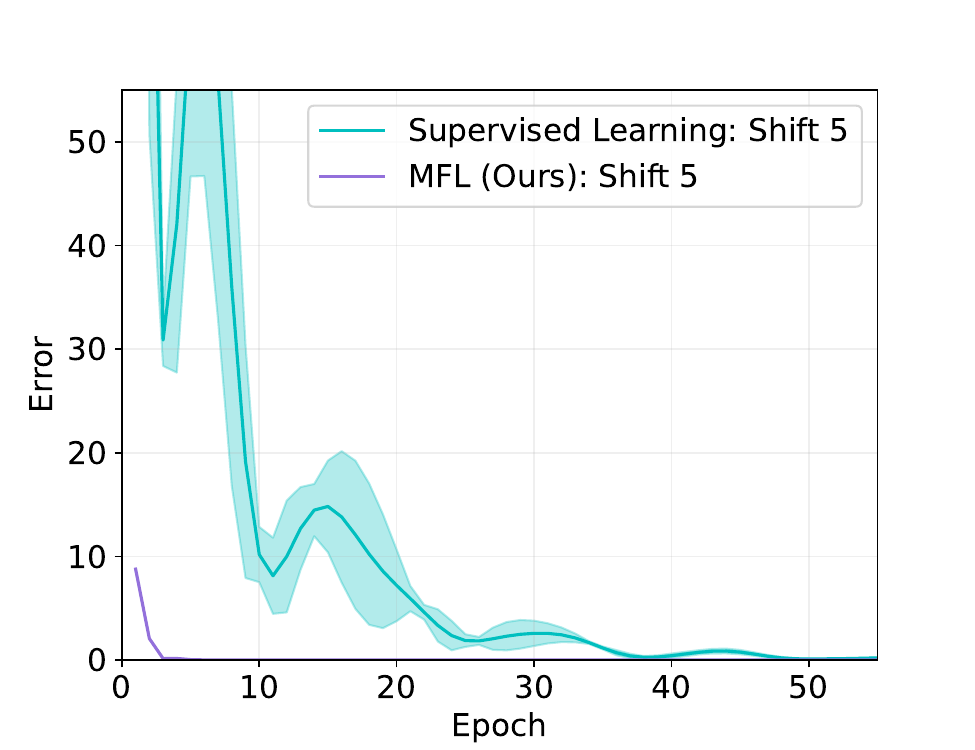}
     } 
 \caption{Robustness evaluation under target shifts.
     }

\label{fig:mfl-sl-target-shift0.1-1-5}
\end{figure}

Robustness is a crucial attribute for deployed models in dynamic environments, where target distributions may shift unpredictably. To evaluate the robustness of MFL, we conducted recipe generation experiments under varying levels of target attack noise. 
As shown in Figure~\ref{exp-fig:robust-evaluation}, our model consistently maintained low output error values across all experiments, even as the magnitude of attack noise increased. Notably, in Experiment 4, where attack noise values escalated significantly to a maximum of 100, the model achieved a remarkably low output error of just 0.04. 
These results demonstrate the model's strong resistance to target noise and its capability to generate accurate outputs under challenging conditions, highlighting its potential for reliable deployment in real-world applications.

In addition, we conducted a robustness evaluation to compare our approach with traditional supervised learning-based recipe generation method. The results, presented in Figure~\ref{fig:mfl-sl-target-shift0.1-1-5}, show the performance of our MFL approach against the supervised learning method across various target shift scenarios in the six-dimensional target space. Specifically, the plots depict the output error per epoch for two target shift conditions: a minor shift of 0.1 (Figure~\ref{fig:mfl-sl-target-shift0.1-1-5a}) and a substantial shift of 5 (Figure~\ref{fig:mfl-sl-target-shift0.1-1-5b}).

In both scenarios, the MFL method (shown in purple) consistently demonstrates rapid convergence with minimal error fluctuations across epochs.
Under the minor target shift condition (Figure~\ref{fig:mfl-sl-target-shift0.1-1-5a}), MFL maintains stability from the initial epochs, effectively mitigating the impact of the shift. 
In contrast, supervised learning (shown in cyan) begins with a higher initial error and converges more gradually, exhibiting notable fluctuations that reflect its sensitivity to even small changes in the target distribution.
For the larger target shift condition (Figure~\ref{fig:mfl-sl-target-shift0.1-1-5b}), MFL again exhibits strong robustness, achieving stability and maintaining a low error rate with minimal variability throughout training. 
Supervised learning, however, starts with a significant error, which persists for multiple epochs before eventually converging. 
The remarkable variability in the supervised learning error trajectory further indicates its difficulty in adapting to substantial shifts.

\begin{table*}[t]
\centering
\vspace{-8pt}
\begin{minipage}[t]{0.43\linewidth}
\centering
\vspace{18pt}
\caption{Chemical Vapor Deposition Target.}
\hspace{-12pt}
\scriptsize{
\begin{adjustbox}{width=0.99\textwidth,center}
\begin{tabular}{lcc}
\hline
{Target Type} & {MFL (Ours)} & {Target Constraints} \\
\hline
Film thickness (center) [nm] & 1047.5 & (100, 2000) \\
Film thickness (edge) [nm] & 1147.3 & (100, 2200) \\
Internal stress [MPa] & 0.047 & (-500, 500) \\
Surface roughness (Ra) [nm] & 5.148 & (0.1, 10) \\
\hline
\end{tabular}
\end{adjustbox}
}
\label{table:chemical-Vapor-Deposition-Target}
\end{minipage}
\hspace{0.04\linewidth}
\begin{minipage}[t]{0.5\linewidth}
\centering 
\caption{Chemical Vapor Deposition Input.}
\vspace{-6pt}
\scriptsize{
\begin{adjustbox}{width=0.93\textwidth,center}
\begin{tabular}{lcc}
\hline
{Input Type} & {MFL (Ours)} & {Input Constraints} \\
\hline
SiH\textsubscript{4} flow rate [sccm] & 317.0320 & (50, 500) \\
NH\textsubscript{3} flow rate [sccm] & 560.5539 & (100, 1000) \\
N\textsubscript{2} flow rate [sccm] & 1288.5685 & (200, 2000) \\
Chamber temperature [°C] & 541.5430 & (300, 750) \\
Chamber pressure [Torr] & 5.1651 & (1, 10) \\
Chamber humidity [\%RH] & 24.6166 & (5, 40) \\
Electrode distance [mm] & 16.7863 & (10, 30) \\
Pre-clean plasma power [W] & 146.6214 & (0, 300) \\
Pre-clean duration [s] & 36.2421 & (0, 60) \\
Wafer rotation speed [rpm] & 1906.4441 & (0, 3000) \\
Process time [s] & 5.05 & 144.5516 \\
\hline
\end{tabular}
\end{adjustbox}
}
\label{table:chemical-Vapor-Deposition-input}
\end{minipage}
\end{table*}

These results highlight MFL’s ability to deliver a stable and efficient learning process across varying target shift levels in high-dimensional space learning. This robustness makes MFL particularly valuable for real-world applications where resilience to evolving data distributions is critical.

\subsection{Beyond Semiconductor Manufacturing}

To further evaluate the generality of our method, we extended our experiments beyond semiconductor recipe generation to two additional manufacturing tasks:  CVD \citep{jones2009chemical, sun2021chemical}\footnote{\url{https://en.wikipedia.org/wiki/Chemical_vapor_deposition}} and wire bonding \citep{chauhan2014copper, zhong2011overview}\footnote{\url{https://en.wikipedia.org/wiki/Wire_bonding}}. In both instances, MFL demonstrated strong performance and high efficiency, requiring only 5 iterations to converge for CVD and 9 for wire bonding. 
Importantly, MFL consistently satisfied all domain-specific constraints on both inputs and targets in our experiment settings, as shown for CVD in Tables~\ref{table:chemical-Vapor-Deposition-Target} and \ref{table:chemical-Vapor-Deposition-input}, and for wire bonding in the supplementary tables in Appendix~\ref{appendix:applications-beyond-semicondutor}. These results highlight the broad applicability of the MFL framework across diverse manufacturing domains.

\begin{figure}[htbp]
  \centering  
  \begin{minipage}[t]{0.435\linewidth}
    \centering        
    \includegraphics[width=\linewidth]{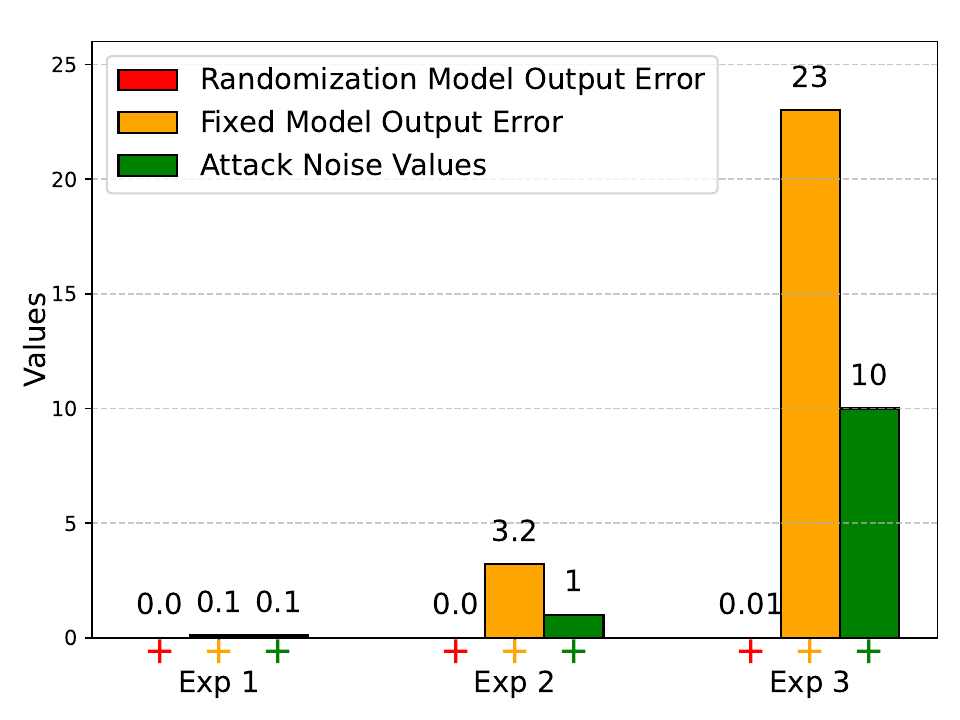}
    \vspace{2pt}
    \caption{\strut Domain randomization ablation: Evaluating the impact of domain randomization on model output error across different attack noise levels.}
    \label{exp-fig:robust-evaluation-ablation-randomization}
  \end{minipage}
  \hspace{0.05\linewidth}   
  \begin{minipage}[t]{0.455\linewidth}  
    \centering
    \includegraphics[width=\linewidth]{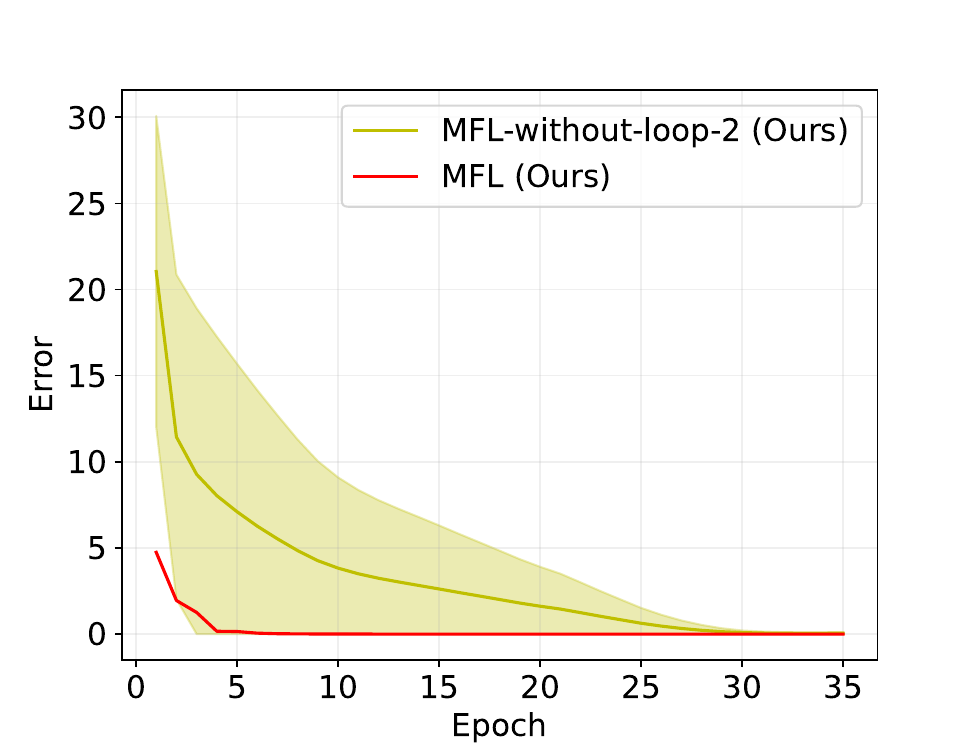}
    \vspace{2pt}
    \caption{\strut Comparison between MFL and MFL without loop 2 training.}
    \label{fig:mfl-compared-with-without-r-loop-2}
  \end{minipage}
\end{figure}

\subsection{Ablation Experiments}
To comprehensively evaluate the performance of our method, we conduct ablation experiments on domain randomization and model pre-training, analyzing the impact of each factor on the overall performance.

\paragraph{Domain Randomization} 
We perform ablation experiments comparing fixed input training with domain-randomized input training to assess robustness under varying attack noise levels. 
Figure~\ref{exp-fig:robust-evaluation-ablation-randomization} presents the results, where the experiments compare three metrics: randomization model output error (red), fixed model output error (orange), and attack noise values (green) across three scenarios (Exp 1, Exp 2, and Exp 3). 
Notably, domain randomization consistently reduces model output error and enhances robustness, particularly under higher noise attack conditions. For instance, in Exp 3, the randomization model achieves significantly lower error compared to the fixed model. These findings underscore the effectiveness of domain randomization in improving model robustness in noisy environments.

\paragraph{Effect of Loop 2 Training in MFL} To assess the effectiveness of loop 2 training in the MFL framework, we conducted an ablation study comparing the model with a variant that omits loop 2. As shown in Figure~\ref{fig:mfl-compared-with-without-r-loop-2}, removing loop 2 significantly degrades performance. The MFL method (red curve) achieves faster and more stable convergence with substantially lower error across training epochs, while the MFL variant without loop 2 (yellow curve) converges more slowly and exhibits higher error. The shaded region denotes standard deviation across multiple runs, further highlighting the improved stability of the complete MFL approach. These results demonstrate that loop 2 training plays a critical role in reducing prediction error and enhancing model robustness.

\paragraph{Comparing Models with and without Pre-training} 

To evaluate the impact of pre-training in Loop A on model performance, we conduct ablation experiments comparing models with pre-training (i.e., including Loop A) and without pre-training (i.e., excluding Loop A), focusing on their training output errors. 
Figure~\ref{fig:mfl-compared-with-without-r} presents the performance of our MFL approach under different training conditions in a six-dimensional target space.
Specifically, Figure~\ref{fig:mfl-compared-with-without-r_a} compares the training error per epoch for MFL with and without pre-training. 
The results show that pre-training significantly accelerates convergence, as evidenced by the consistently lower errors throughout the training process and the reduced variability. 
Notably, even without pre-training, the MFL method demonstrates a clear advantage over the traditional supervised learning approach, achieving lower errors and greater stability, as shown in Figure~\ref{fig:mfl-compared-with-without-r_b}.

 \begin{figure}[htbp!]
 \centering
 \subcaptionbox{\label{fig:mfl-compared-with-without-r_a}}
  {
\includegraphics[width=0.47\linewidth]{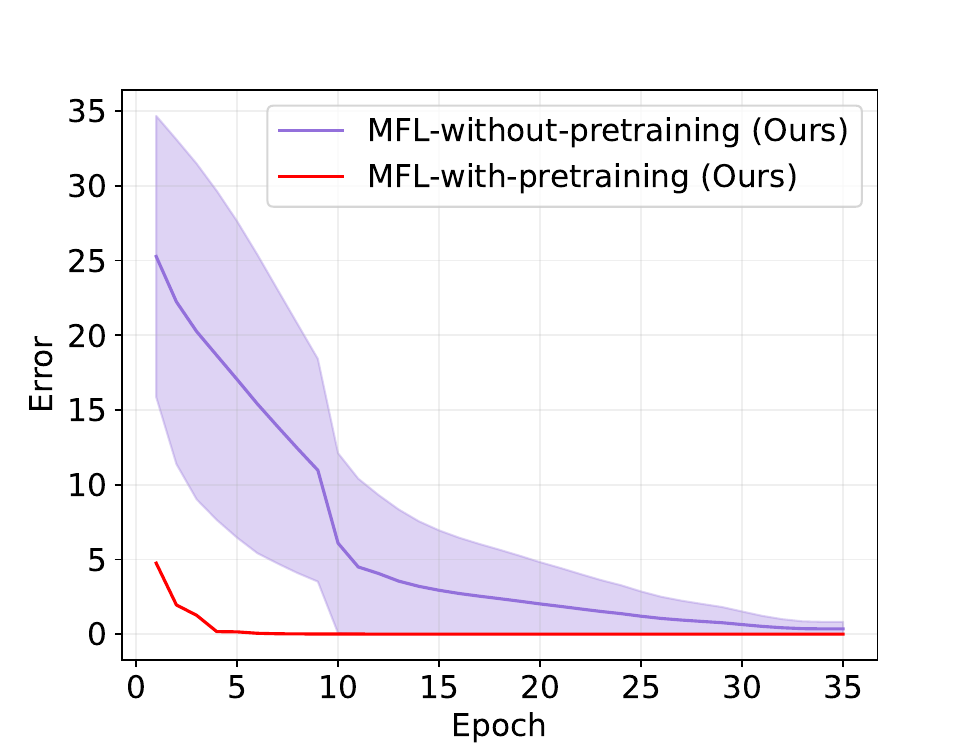}
}    
\subcaptionbox{\label{fig:mfl-compared-with-without-r_b}}
  {
\includegraphics[width=0.47\linewidth]{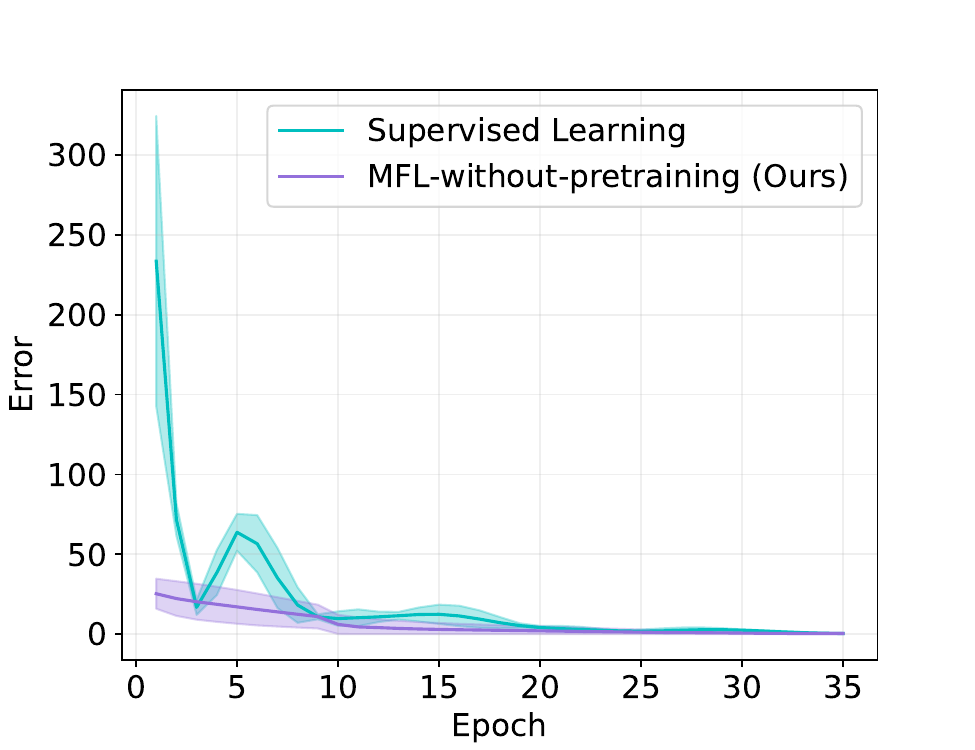}
}   
 	\caption{\normalsize MFL without pre-training is compared to MFL with pre-training (a) and the supervised learning method (b).
 	} 
  \label{fig:mfl-compared-with-without-r}
 \end{figure} 
\section{Conclusion}

In this study, we introduced MFL, a test-time optimization framework designed for recipe generation in semiconductor manufacturing. MFL employs a lightweight reverse emulator to efficiently identify optimal input values for deployed models, such as physical or foundation models, where retraining is infeasible or prohibitively expensive. This enables the generation of outputs that closely match desired targets while maintaining computational efficiency. To validate the effectiveness of the MFL framework, we conducted comprehensive experiments, including robustness evaluations, ablation studies, and comparisons against SOTA baselines and human experts. We also demonstrated the generality of our approach through applications beyond semiconductor manufacturing, including chemical and electronic domains. Experimental results show that MFL consistently outperforms both SOTA methods and human-driven approaches, highlighting its potential to improve accuracy and efficiency in diverse applications. Particularly, MFL offers a promising foundation for advancing the development of next-generation semiconductor technologies—critical enablers for future innovations in AI and other emerging fields.

\bibliography{mfl}
\bibliographystyle{plain}

\newpage
\appendix

\textbf{\Large Appendix}

\section{More Details on Plasma Etching Process in Semiconductor Manufacturing}\label{appendix:plasma_etching}

Modern computer chips feature complex circuit patterns that conduct electricity through millions and even billions of transistors and other types of devices. Some transistors are no larger than tens of silicon atoms. 
To create circuit patterns on nanometer scales, a precise process that removes silicon and other materials from the wafer surface is typically achieved with plasma etching. 
Plasma etching is a critical process in semiconductor device manufacturing. The method uses plasma to selectively remove material from a substrate's surface. This process is essential for creating precise, high-resolution features on semiconductor wafers.
A plasma etching process involves the generation of plasma in a specially designed chamber, where certain chemically reactive species and ions are generated, and the etching of the wafer according to lithographical patterns.
Physical-chemical reactions in the etching process need to be accurately tuned and precisely controlled to form on-wafer structures with specific depths and shapes.

Figure~\ref{fig:semiconductor-manufacturing-overview-nature} provides a schematic overview of the plasma etching process, a critical step in semiconductor manufacturing that demands nanometer-scale precision. The process involves optimizing a set of recipe parameters to control the etching results and ensure profile fidelity under stringent constraints.

\textbf{Incoming Photoresist Mask:} The leftmost panel of Figure~\ref{fig:semiconductor-manufacturing-overview-nature} depicts the wafer structure prior to etching. It includes an oxide layer covered by a patterned photoresist mask with a thickness of approximately 750 nm. The target feature has a critical dimension (CD) of 200 nm, emphasizing the precision required during etching.

\textbf{Recipe Parameters:} The central portion of the figure outlines the key input variables for the plasma etching process:
\begin{itemize}[leftmargin=*]
    \item \textbf{Pressure and Gas Flows:} Controlled flows of gases such as \textit{Ar}, \textit{${C}_4{F}_6$}, \textit{${C}_4{F}_8$}, \textit{${CH}_3{F}$}, and \textit{${O}_2$} determine plasma chemistry and material selectivity.
    \item \textbf{Plasma Powers:} Power 1 and Power 2 regulate the RF energy delivered to sustain and modulate the plasma.
    \item \textbf{Pulsing Parameters:} Duty cycle and frequency control temporal variations in plasma intensity, influencing etch anisotropy and uniformity.
    \item \textbf{Wafer Temperature:} Maintains thermal stability to minimize undesirable effects, such as polymer deposition or trench bowing.
\end{itemize}

\textbf{Simulator Outputs:} The simulator evaluates the effect of recipe parameters on critical performance metrics:
\begin{itemize}[leftmargin=*]
    \item \textbf{Etch Depth and Rate:} Quantify the depth and speed of material removal.
    \item \textbf{Mask Remaining:} Measures the residual thickness of the photoresist after etching.
    \item \textbf{Top CD and Bow CD:} Represent the critical dimension at the surface and along the trench profile.
    \item \textbf{$\Delta$CD:} Captures the variation in CD from top to bottom.
\end{itemize}

\textbf{Target and Deviant Etch Profiles:} The rightmost panels in Figure~\ref{fig:semiconductor-manufacturing-overview-nature} illustrate desired and undesired etching outcomes:
\begin{itemize}[leftmargin=*]
    \item \textbf{Target Profile:} Features uniform etch depth and vertical sidewalls with minimal CD variation, meeting the design specifications.
    \item \textbf{Deviant Profiles:} Highlight potential issues such as trench wall bowing, polymer deposition, or incomplete etching caused by process variability.
\end{itemize}

The plasma etching process is inherently complex due to the interactions between recipe parameters, plasma dynamics, and process variability. Optimizing these parameters is critical to ensuring consistent, high-quality manufacturing outcomes, particularly in applications where nanometer-level precision is required. 
Table \ref{appendix-table:input_parameters-constraints} outlines the input parameter constraints for the etching process, including the input ranges specified by human engineers. Here, ``SE" represents senior engineers, and ``JE" represents junior engineers. 
Table \ref{appendix-table:process_output_targets} presents the output targets for the etching process, including the criteria for three categories: ``Meets Target" (ideal target), ``Close to Target," and ``Far from Target." These categories are defined based on specific ranges for key metrics such as etch depth, etch rate, mask remaining, top CD, $\Delta$CD, and bow CD.

\begin{table}[!htbp]
\centering
\caption{Input parameter search ranges \citep{kanarik2023human}.}
\scriptsize{
	\begin{adjustbox}{width=0.95\textwidth,center}
\begin{tabular}{@{}lccccccc@{}}
\toprule
\textbf{Input Parameters} & \textbf{Unconstrained} & \textbf{SE1} & \textbf{SE2} & \textbf{SE3} & \textbf{JE1} & \textbf{JE2} & \textbf{JE3} \\ \midrule
Pressure (mT)             & 5 -- 120              & 12 -- 30     & 12 -- 30     & 5 -- 23      & 5 -- 23      & 20 -- 38     & 12 -- 30     \\
Power 1 (W)               & 0 -- 29,000           & 4,000 -- 15,000 & 4,000 -- 15,000 & 4,000 -- 15,000 & 4,000 -- 15,000 & 14,000 -- 25,000 & 4,000 -- 15,000 \\
Power 2 (W)               & 0 -- 10,000           & 1,000 -- 7,000  & 1,000 -- 7,000  & 0 -- 6,000      & 1,000 -- 7,000  & 2,000 -- 8,000  & 0 -- 6,000      \\
Ar Flow (sccm)            & 0 -- 1,000            & 100 -- 400     & 0 -- 300       & 0 -- 300       & 0 -- 300       & 300 -- 600     & 0 -- 300       \\
C\textsubscript{4}F\textsubscript{8} Flow (sccm) & 0 -- 100       & 20 -- 60       & 20 -- 60       & 10 -- 50       & 0 -- 40        & 40 -- 80       & 20 -- 60       \\
C\textsubscript{4}F\textsubscript{6} Flow (sccm) & 0 -- 100       & 22 -- 66       & 15 -- 59       & 10 -- 54       & 0 -- 44        & 48 -- 96       & 12 -- 56       \\
CH\textsubscript{4} Flow (sccm) & 0 -- 20         & 0 -- 5         & 0 -- 5         & 0 -- 5         & 7.5 -- 12.5    & 3 -- 8         & 15 -- 20       \\
O\textsubscript{2} Flow (sccm) & 0 -- 50          & 20 -- 50       & 20 -- 50       & 10 -- 40       & 10 -- 40       & 20 -- 50       & 20 -- 50       \\
Pulse Duty Cycle (\%)     & 10 -- 100             & 20 -- 60       & 30 -- 70       & 10 -- 50       & 10 -- 50       & 30 -- 70       & 10 -- 50       \\
Pulse Frequency (Hz)      & 500 -- 2,000          & 1,000          & 1,000          & 1,000          & 1,000          & 1,000          & 1,000          \\
Temperature (°C)          & -15 -- 80             & 20 -- 45       & 10 -- 35       & 30 -- 55       & 20 -- 45       & 25 -- 50       & 15 -- 40       \\ \bottomrule
\end{tabular}
\end{adjustbox}
}
\label{appendix-table:input_parameters-constraints}
\end{table}

\begin{table*}[!htbp]
\centering
\caption{Etching Process output targets \citep{kanarik2023human}.}
\scriptsize{
	\begin{adjustbox}{width=0.92\textwidth,center}
\begin{tabular}{@{}lcccccc@{}}
\toprule
\textbf{Category}       & \textbf{Etch depth}     & \textbf{Etch rate} & \textbf{Mask remaining} & \textbf{Top CD}          & \textbf{$\Delta$CD}          & \textbf{Bow CD}        \\ \midrule
Meets target            & 2250 to 2750           & $\geq$100          & $\geq$350               & 190 to 210               & -15 to 15                   & 190 to 210             \\ \midrule
Close to target         & 2000 to 2250 or        & 70--100            & 300 to 350              & 160 to 190 or            & -60 to 15 or                & 160 to 190 or          \\
                        & 2750 to 3000           &                    &                         & 210 to 240               & 15 to 60                    & 210 to 240             \\ \midrule
Far from target         & $<$2000 or $>$3000     & $<$70              & $<$300                  & $<$160 or $>$240         & $<$-60 or $>$60             & $<$160 or $>$240       \\ \bottomrule
\end{tabular}
\end{adjustbox}
}
\label{appendix-table:process_output_targets}
\end{table*}

\section{Convergence of Algorithm \ref{alg:algorithm-mfl}}\label{app: convergence}
In this section, we discuss the convergence property of Algorithm \ref{alg:algorithm-mfl}.
For simplicity, we assume that the emulator model $\E$ perfectly approximates the machine model $\M$.
Then, both Loop A and Loop B in Algorithm \ref{alg:algorithm-mfl} are essentially minimizing the following loss function with respect to parameter $\theta$:
\begin{equation}\label{eq: loss_def}
\LL(\theta) = \dfrac{1}{n^\prime} \sum_{j=1}^{n^\prime} \norm{z_j^\prime - \M\big(\R_{\theta}(z_j^\prime)\big)}^2.
\end{equation}
Below, we do not distinguish between Loop A and Loop B, and consider a total of $T$ iterations of Algorithm \ref{alg:algorithm-mfl}.
Since $\R_\theta$ is parameterized by a neural network, we generally cannot ensure the convergence to the global minimum of $\LL(\theta)$.
In the following theorem, we demonstrate that Algorithm \ref{alg:algorithm-mfl} converges to a stationary point of $\LL(\theta)$, if $\LL$ is Lipschitz smooth in $\theta$ (Note that this property naturally holds true if $\M$ and $\R_\theta$ exhibit certain smoothness properties).
When there are sufficiently many samples (i.e., $n^\prime$ is large), we can expect that the achieved stationary point corresponds to a reasonably good reverse model.
\\

\begin{theorem}\label{thm: convergence}
Suppose that $\LL(\theta)$ is $L$-Lipschitz smooth. Then, if the learning rate satisfies that $\alpha^t \equiv \alpha < 1/L$ for all $t = 0,\dots, T-1$, it holds that
\begin{equation}\label{eq: stationarity}
\sum_{t=0}^{T-1} \norm{\nabla_\theta \LL(\theta^t)}^2\leq \dfrac{\LL(\theta^0) - \LL(\theta^T)}{\alpha (1- \alpha L/2)}.
\end{equation}
Therefore, it holds that $\lim_{t\rightarrow \infty} \norm{\nabla_\theta \LL(\theta^t)}^2 = 0$.
\end{theorem}
\begin{proof}[Proof of Theorem \ref{thm: convergence}]

The proof is based on the basic optimization theory.
By the smoothness of $\LL(\theta)$, it holds that
\begin{equation}
\LL(\theta^{t+1}) \leq \LL(\theta^t) + \nabla_\theta\LL(\theta^t)^\top(\theta^{t+1} - \theta^t) + \dfrac{L}{2}\norm{\theta^{t+1} - \theta^t}^2.
\end{equation}
Then, we substitute the relation $\theta^{t+1}=\theta^t-\alpha \cdot \nabla_\theta \mathcal{L}\left(\theta^t\right)$ to the above equation, which yields that
\begin{equation}
\LL(\theta^{t+1}) \leq \LL(\theta^t) - \alpha \norm{\nabla_\theta\LL(\theta^t)}^2 + \dfrac{\alpha^2 L }{2}\norm{\nabla_\theta\LL(\theta^t)}^2.
\end{equation}
Equivalently, this is can be written as
\begin{equation}\label{eq: to_be_tele}
\alpha(1-\alpha L /2)\norm{\nabla_\theta\LL(\theta^t)}^2 \leq   \LL(\theta^t) - \LL(\theta^{t+1})
\end{equation}
Finally, Eq. \eqref{eq: stationarity} can be obtained by performing a telescoping sum on Eq. \eqref{eq: to_be_tele} from $t=0$ to $t=T-1$.
\end{proof}

\section{Baseline Method for Comparative Analysis}
\label{appendix-sec:random-search-optimization}

As described in Algorithm \ref{algorithm:random-search-local-refinement}, the \textbf{Large-Scale Random Search with Local Refinement (LSRS-LR)} algorithm is also designed to efficiently optimize inputs for a pre-trained emulator model to achieve a target output. 

The LSRS-LR algorithm operates in two stages: 
\begin{itemize}[leftmargin=*]
    \item First, a \textit{large-scale random search} generates $N$ random candidate inputs within predefined bounds, evaluates their outputs using the emulator model, and selects the top $K$ candidates based on the lowest mean squared error (MSE) loss.
    \item In the second stage, \textit{gradient-based local refinement}, each selected candidate is expressed as a neural network parameter and undergoes iterative optimization using the Adam optimizer to further minimize the prediction error while ensuring the inputs remain within valid bounds. 
\end{itemize}
The candidate with the lowest final loss is selected as the output of the algorithm, denoted as $x^*$. By balancing exploration and exploitation, this approach enhances efficiency in complex optimization tasks such as semiconductor manufacturing.

To evaluate the effectiveness of LSRS-LR, we conduct a comparative experiment between the {MFL} algorithm (Algorithm~\ref{alg:algorithm-mfl}) and the {LSRS-LR} algorithm (Algorithm~\ref{algorithm:random-search-local-refinement}). As shown in Table~\ref{tab:etch_data_semiconductor_manufacturing_LSRS-LR}, the ``Target range" row defines the acceptable operational ranges for each performance metric. The results demonstrate that MFL outperforms LSRS-LR, particularly in \textit{etch depth}, \textit{mask remaining}, and \textit{critical dimension (CD) control}. Notably, MFL consistently maintains values within the target range across all parameters, whereas LSRS-LR exhibits deviations in mask remaining and etch depth.
These findings underscore the superior precision and accuracy of MFL, reinforcing its effectiveness in semiconductor manufacturing optimization.

The experimental settings for Algorithm~\ref{algorithm:random-search-local-refinement} are as follows: the parameters are set to $N=100$, $K=10$, $T=200$, and $\eta=0.01$. The search space bound $[x_{\min}, x_{\max}]$ is provided in Table~\ref{appendix-table:algorithm-hyparameter-experiments-training}, and the target output, $y_{\text{target}}$, is specified as \([2260, 110, 360, 200, 10, 200]\).

\begin{algorithm}[htbp]
\caption{Large Scale Random Search with Local Refinement (LSRS-LR)}
\label{algorithm:random-search-local-refinement}
\begin{algorithmic}[1]
\STATE \textbf{Input:} Pre-trained emulator model $\E$, target output $y_{\text{target}}$, search space bound $[x_{\min}, x_{\max}]$, number of initial samples $N$, number of selected candidates $K$, refinement iteration number $T$, leanring rate $\eta$.
\STATE \textcolor{gray}{ \textbf{// Step 1: Large-scale random search.}}
\STATE Generate $N$ random input samples: $x_i \sim \mathcal{U}(x_{\min}, x_{\max})$, for $i = 1, \dots, N$.
\STATE Compute emulator ouput: $y_i = \E(x_i)$, for $i = 1, \dots, N$.
\STATE Compute MSE losses: $\LL_i = \norm{y_i - y_{\text{target}}}^2$, for $i = 1, \dots, N$.
\STATE Select the best $K$ candidates with minimum MSE losses, denoted as $\{x_{i_k}\}_{k=1}^{K}$.
\STATE \textcolor{gray}{ \textbf{// Step 2: Gradient-based local refinement for every candidate $x_{i_k}$.}}
\FOR{$k=1$ to $K$}
    \STATE Initialize $x \leftarrow x_{i_k}$ and express $x$ as a neural network parameter.
    \FOR{$t=1$ to $T$}
        \STATE Compute emulator output: $y = \mathcal{E}(x)$.
        \STATE Compute MSE loss: $\LL(x) = \norm{ y - y_{\text{target}}}^2 =\norm{\E( x) - y_{\text{target}}}^2 $.
        \STATE Update $x$ with learning rate $\eta$.
        \STATE Project $x$ to the feasible search space: ${x} = \text{clip}(x, x_{\min}, x_{\max})$.
    \ENDFOR
    \STATE Store the refined candidate $\hat x_{i_k} = x$, and the associated loss $\LL(\hat x_{i_k})$.
\ENDFOR
\STATE \textbf{Output:} The refined input with the smallest loss $x^* = \underset{k \in  \set{1,\dots,K}}{\argmin} \mathcal{L}(\hat x_{i_k})$. 
\end{algorithmic}
\end{algorithm}


\begin{table}[htbp]
    \centering
    \caption{Performance comparison between MFL (Algorithm \ref{alg:algorithm-mfl}) and LSRS-LR (Algorithm \ref{algorithm:random-search-local-refinement}). Values that do not meet the target requirements are highlighted in \textcolor{red}{red}, while those that satisfy the target requirements are highlighted in \textcolor{blue}{blue}.}
    \scriptsize{
	\begin{adjustbox}{width=0.98\textwidth,center}
    \begin{tabular}{ccccccc}
        \hline
        \textbf{Category} & \textbf{Etch depth} & \textbf{Etch rate} & \textbf{Mask remaining} & \textbf{Top CD} & \textbf{$\Delta$CD} & \textbf{Bow CD} \\
        \hline
        Target range & {2250 to 2750} & $\geq$100 & $\geq$350 & 190 to 210 & -15 to 15 & 190 to 210 \\
        LSRS-LR (Ours) & \textcolor{red}{2242.7803} & \textcolor{blue}{138.3506} & \textcolor{red}{340.0843} & \textcolor{blue}{194.1832} & \textcolor{blue}{3.1649} & \textcolor{blue}{192.0442} \\
        MFL (Ours) & \textcolor{blue}{2255.5466} & \textcolor{blue}{109.8992} & \textcolor{blue}{358.9485} & \textcolor{blue}{198.7954} & \textcolor{blue}{10.0430} & \textcolor{blue}{198.5240} \\
        \hline
    \end{tabular}
    \end{adjustbox}
    }    
    \label{tab:etch_data_semiconductor_manufacturing_LSRS-LR}
\end{table}

\section{Applications Beyond Semiconductor Manufacturing}
\label{appendix:applications-beyond-semicondutor}

To further evaluate the generality of the MFL framework, we applied it to wire bonding\footnote{\url{https://en.wikipedia.org/wiki/Wire_bonding}}, a critical process in electronic packaging. As shown in Table~\ref{table:wire-bonding-target} and Table~\ref{table:wire-bonding-input}, MFL successfully identified input parameters that satisfy strict domain-specific constraints while producing outputs that meet precise target specifications. Table~\ref{table:wire-bonding-target} presents the predicted target values generated by MFL, including pull strength and bonding offsets, all of which fall well within the specified acceptable ranges. Similarly, Table~\ref{table:wire-bonding-input} reports the corresponding input parameters discovered by MFL—such as bonding pressure, time, temperature, and material dimensions—all adhering to the process constraints. Notably, MFL converged to a valid solution in only 9 iterations, demonstrating both efficiency and reliability in this new application domain.

These results further demonstrate the broad applicability and robustness of MFL across diverse applications beyond semiconductor etching, highlighting its potential for widespread deployment in real-world industrial settings.

\begin{table}[h]
\centering
\caption{Wire Bonding Target}
\resizebox{0.65\textwidth}{!}{
\begin{tabular}{lcc}
\hline
\textbf{Target Type} & \textbf{MFL (Ours)} & \textbf{Target Constraints} \\
\hline
Pull strength [gf] & 15.005 & (5, 25) \\
Bonding x-offset [µm] & -0.0001 & (-20, 20) \\
Bonding y-offset [µm] & -0.0002 & (-20, 20) \\
\hline
\end{tabular}
}
\label{table:wire-bonding-target}
\end{table}

\begin{table}[h]
\centering
\captionof{table}{Wire Bonding Input}
\resizebox{0.65\textwidth}{!}{
\begin{tabular}{lcc}
\hline
\textbf{Input Type} & \textbf{MFL (Ours)} & \textbf{Input Constraints} \\
\hline
Bonding pressure [gf] & 98.6096 & (20, 120) \\
Bonding time [ms] & 19.9200 & (1, 30) \\
Temperature [°C] & 240.9897 & (100, 300) \\
Wire diameter [µm] & 28.0635 & (15, 33) \\
Wire length [mm] & 3.2536 & (0.5, 5.0) \\
Pad diameter [µm] & 88.7107 & (50, 150) \\
\hline
\end{tabular}
}
\label{table:wire-bonding-input}
\end{table}

\clearpage
\section{Experimental Setup}
\label{appendix:experiment-settings}

We conducted our experiments on a server equipped with an AMD Ryzen 7 2700X CPU and an NVIDIA GeForce RTX 2060 GPU. The operating system used is Ubuntu 20.04.3 LTS. Table~\ref{appendix-table:algorithm-hyparameter-experiments-training} provides a summary of the key parameters used in our experiments.

\begin{table}[!htbp]
 \renewcommand{\arraystretch}{1.2}
  \centering
  \caption{Key parameters used in the experiments. }
\label{appendix-table:algorithm-hyparameter-experiments-training}
  \begin{threeparttable}
    \begin{tabular}{cc|cc}
    \toprule
    Parameters & value & Parameters & value \\
    \midrule
    lr, i.e., $\alpha_1$ & 0.01             &       epoch & 700    \\  
           hidden layer dim & 64         & neural network  & MLP  \\ 
           Pressure (mT)   & (5, 38)          & Power 1 (W)  & (4000, 25000)  \\
           Power 2 (W) &     (0, 8000)      & Ar Flow (sccm)  & (0, 600)  \\
           C\textsubscript{4}F\textsubscript{8} Flow (sccm) &     (0, 80)     & C\textsubscript{4}F\textsubscript{6} Flow (sccm)  & (0, 66)  \\
           CH\textsubscript{4} Flow (sccm) &    (0, 20)     & O\textsubscript{2} Flow (sccm)  & (0, 50)  \\
           Pulse Duty Cycle (\%) &     (10, 70)      & Pulse Frequency (Hz)  & (950, 1050))  \\
           Temperature (\textdegree C) &   (10, 55)    & num dimensions  & 11  \\
           gradient threshold, i.e., $\delta$ & 0.9 & $T_0$ & 1150 \\
           $\tau_0$ & 150 & $\tau$ & 200 \\ 
           $T$ & 1200 & $\alpha_2$ & 0.99*0.01 \\
    \bottomrule
    \end{tabular}    
    \end{threeparttable}
\end{table}


\end{document}